%% file: sample-sigconf.tex
\documentclass[sigconf,natbib=true,screen,anonymous=false]{acmart}
\usepackage{xspace,balance,tabularx,multirow}
\usepackage{flushend}
\usepackage{tikz}
\usepackage{pgfplots}
\usepackage{tikz}
\usepackage{pgfplots}
\usepgfplotslibrary{groupplots} %
\usetikzlibrary{matrix, positioning, calc} %
\pgfplotsset{compat=1.18}

\usepackage{algorithmic}
\usepackage{boldline}
\usepackage{makecell}

\usepackage{tcolorbox}

\definecolor{pltBlue}{HTML}{1F77B4}   %
\definecolor{pltOrange}{HTML}{FF7F0E} %
\definecolor{pltGreen}{HTML}{2CA02C}  %
\pgfplotsset{compat=1.16}
\usetikzlibrary{patterns}
\usepackage{subfig}
\usepackage[ruled, vlined, linesnumbered]{algorithm2e}
\SetKw{continue}{continue}
\usepackage{xcolor}
\usepackage{colortbl}
\usepackage{bbold}
\SetKwComment{Comment}{$\triangleright$\ }{}
\usepackage{enumitem}
\usepackage{tablefootnote}
\usepackage{upgreek,textgreek}
\usepackage{pifont}%
\usepackage[noabbrev]{cleveref}
\usepackage{titlecaps}
\usepackage{lipsum}
 
\usepackage{bm}

\pgfplotsset{every tick label/.append style={font=\tiny}}

\newlength{\starsize}
\newlength{\starspread}
\tikzset{starsize/.code={\setlength{\starsize}{#1}},
         starspread/.code={\setlength{\starspread}{#1}}}
\tikzset{starsize=1mm,
         starspread=3mm}
\pgfdeclarepatternformonly[\starspread,\starsize]%
  {my fivepointed stars}%
  {\pgfpointorigin}%
  {\pgfqpoint{\starspread}{\starspread}}%
  {\pgfqpoint{\starspread}{\starspread}}%
  {%
   \pgftransformshift{\pgfqpoint{\starsize}{\starsize}}
   \pgfpathmoveto{\pgfqpointpolar{18}{\starsize}}
   \pgfpathlineto{\pgfqpointpolar{162}{\starsize}}
   \pgfpathlineto{\pgfqpointpolar{306}{\starsize}}
   \pgfpathlineto{\pgfqpointpolar{90}{\starsize}}
   \pgfpathlineto{\pgfqpointpolar{234}{\starsize}}
   \pgfpathclose%
   \pgfusepath{fill}
  }

\input{tex/macro}

\AtBeginDocument{%
  \providecommand\BibTeX{{%
    \normalfont B\kern-0.5em{\scshape i\kern-0.25em b}\kern-0.8em\TeX}}}

\setcopyright{acmcopyright}
\copyrightyear{2018}
\acmYear{2018}
\acmDOI{XXXXXXX.XXXXXXX}

\acmPrice{15.00}
\acmISBN{978-1-4503-XXXX-X/18/06}

\acmSubmissionID{1447}

\begin{document}

\title{Adaptive Graph Refinement and Label Propagation with LLMs for Cost-Effective Entity Resolution}
\subtitle{Technical Report}

\author{Hongtao Wang}
\affiliation{%
  \institution{Hong Kong Baptist University}
  \city{Hong Kong}
  \country{China}
}
\email{cshtwang@comp.hkbu.edu.hk}
\orcid{0009-0002-1279-4357}

\author{Renchi Yang}
\authornote{Corresponding Author}
\affiliation{%
  \institution{Hong Kong Baptist University}
  \city{Hong Kong}
  \country{China}
}
\email{renchi@hkbu.edu.hk}
\orcid{0000-0002-7284-3096}

\author{Haoran Zheng}
\affiliation{%
  \institution{Hong Kong Baptist University}
  \city{Hong Kong}
  \country{China}
}
\email{cshrzheng@comp.hkbu.edu.hk}
\orcid{0009-0002-4769-3716}

\author{Xiangyu Ke}
\affiliation{%
  \institution{Zhejiang University}
  \city{Hangzhou}
  \country{China}
}
\email{xiangyu.ke@zju.edu.cn}
\orcid{0000-0001-8082-7398}

\settopmatter{printfolios=true}

\renewcommand{\shortauthors}{Wang et al.}

\begin{abstract}
{\em Dirty entity resolution} (ER), which identifies records referring to the same real-world entity from a single, messy dataset, is a fundamental task in data management and mining.
However, the dominant {\em blocking-matching-clustering} paradigm for ER suffers from critical flaws. Its cascaded, decoupled workflow essentially produces a static, sparse graph plagued by missing edges (due to blocking failures) and noisy links (due to matching errors), causing error propagation and yielding suboptimal clusters, particularly when rigid transitivity is imposed in the clustering.

We contend that matching and clustering are fundamentally synergistic, both optimizing for the construction of an ideal entity graph. Building upon this insight, we propose \algo{}, a unified framework that integrates these steps into an iterative probabilistic label propagation process over a global, evolving graph.
Unlike disjoint blocking, \algo{} refines the graph structure and labels dynamically by adaptively integrating ``weak but cheap'' signals from graph propagation with ``strong but expensive'' LLM-based pairwise queries. 
For higher cost-effectiveness, we formulate the signal selection as a constrained optimization problem maximizing cumulative marginal gain under a query budget, solved via our greedy algorithm with provable theoretical guarantees. Our extensive experiments over eight benchmark datasets demonstrate that \algo{} is consistently superior to state-of-the-art cascaded pipelines.
\end{abstract}

\maketitle

\input{tex/introduction}

\input{tex/relatedwork}

\input{tex/preliminary}

\input{tex/method}

\input{tex/experiments}
\section{Conclusion}
In this paper, we propose \algo{}, a unified framework for cost-effective ER that synergizes graph-based label propagation and LLMs. Our method iteratively refines graph topology and entity labels, featuring a theoretically grounded \emph{adaptive signal selection} mechanism modeled as an \emph{Online Knapsack Problem} to maximize marginal value gain under strict budget constraints. We also demonstrate how to effectively shatter the recall ceiling of static blocking by dynamically integrating ``weak but free'' structural signals with ``strong but expensive'' LLM reasoning. Experiments on eight benchmark datasets show that \algo{} consistently achieves state-of-the-art performance in both accuracy and monetary cost-effectiveness compared to strong baselines like \texttt{BatchER} and \texttt{LLM-CER}, validating our integrated, budget-aware approach.

\begin{acks}
Renchi Yang is supported by the Guangdong and Hong Kong Universities ``1+1+1'' Joint Research Collaboration Scheme, project No.: 2025A0505000002, the NSFC (No. 62302414), the Hong Kong RGC ECS grant (No. 22202623), and YCRG (No. C2003-23Y).
Xiangyu Ke is supported by the NSFC (No. 62502434) and the Ningbo Yongjiang Talent Introduction Programme (2022A-237-G).
\end{acks}

\balance
\bibliographystyle{ACM-Reference-Format}
\bibliography{sample-base}

\pagebreak
\appendix
\input{tex/add-relatedwork}

\input{tex/proof}

\input{tex/add-exp}

\end{document}

%% file: tex/macro.tex
\newcommand{\argmax}[1]{\underset{#1}{\operatorname{arg}\,\operatorname{max}}\;}

\makeatletter
\newcommand*\bigcdot{\mathpalette\bigcdot@{.5}}
\newcommand*\bigcdot@[2]{\mathbin{\vcenter{\hbox{\scalebox{#2}{$\m@th#1\bullet$}}}}}
\makeatother

\newcommand{\eat}[1]{}

\newcommand{\stitle}[1]{\vspace*{0.5em}\noindent{\underline{\bf #1.\/}}}

\newcommand{\G}{\mathcal{G}\xspace}
\newcommand{\N}{\mathcal{N}\xspace}

\newcommand{\rvec}{\vec{\mathbf{r}}\xspace}

\newcommand{\C}{\mathcal{C}\xspace}

\newcommand{\algo}{\texttt{Alper}\xspace}

\SetKwComment{tcp}{\(\triangleright\)\ }{}

\newcolumntype{?}[1]{!{\vrule width #1}}

\newenvironment{customlegend}[1][]{%
    \begingroup
    \csname pgfplots@init@cleared@structures\endcsname
    \pgfplotsset{#1}%
}{%
    \csname pgfplots@createlegend\endcsname
    \endgroup
}%

\def\addlegendimage{\csname pgfplots@addlegendimage\endcsname}

\makeatletter
\newcommand\footnoteref[1]{\protected@xdef\@thefnmark{\ref{#1}}\@footnotemark}
\makeatother

\let\oldnl\nl%
\newcommand{\nonl}{\renewcommand{\nl}{\let\nl\oldnl}}%

\SetKwComment{Comment}{/* }{ */}

\makeatletter 
\g@addto@macro{\@algocf@init}{\SetKwInOut{Parameter}{Parameters}} 
\makeatother

\definecolor{myred}{HTML}{fd7f6f}
\definecolor{myred_new}{HTML}{D8D8D8}
\definecolor{myred_new2}{HTML}{D7191C}
\definecolor{myblue}{HTML}{7eb0d5}
\definecolor{mygreen}{HTML}{b2e061}
\definecolor{mypurple}{HTML}{bd7ebe}
\definecolor{myorange}{HTML}{ffb55a}
\definecolor{myyellow}{HTML}{ffee65}
\definecolor{mypurple2}{HTML}{beb9db}
\definecolor{mypink}{HTML}{fdcce5}
\definecolor{mycyan}{HTML}{8bd3c7}

\definecolor{myblue2}{HTML}{115f9a}
\definecolor{myred2}{HTML}{c23728}

%% file: tex/introduction.tex
\section{Introduction}
{\em Entity Resolution} (ER), broadly defined as identifying records referring to the same real-world entity, is a cornerstone of data quality management~\cite{fellegi1969theory, elmagarmid2006duplicate, hassanzadeh2009framework}. It underpins a wide spectrum of critical downstream applications, ranging from knowledge graph construction~\cite{dong2014knowledge} and e-commerce product management~\cite{kannan2011matching}, to CRM~\cite{konda2016magellan}, financial fraud detection~\cite{phua2010comprehensive}, and census data integration~\cite{winkler2006overview}. 
While traditional {\em record linkage}~\cite{christophides2020overview} (or {\em clean-clean} ER) typically focuses on linking records across two aligned and relatively clean sources, this paper targets the more challenging problem of {\em dirty} ER~\cite{christen2011survey} (or {\em deduplication}), which aims to detect duplicates within a single, messy dataset containing an unknown number of entities where data quality is severely compromised by typographical errors, missing values, and inconsistent representations. This setting is inherently more complex, as it requires not only resolving pairwise ambiguities but also deducing the global cluster structure under strict transitivity constraints, making it tenaciously challenging to balance precision, recall, and cost~\cite{chai2016cost, demartini2012zencrowd, wang2012crowder, whang2013question}. 

Existing research landscape for {\em dirty} ER generally evolves along three paradigms: {\em learning-based methods}~\cite{ebraheem2017deeper, li2020deep, ge2021collaborem}, {\em crowdsourcing-based} strategies~\cite{wang2012crowder, demartini2012zencrowd, verroios2015entity}, and the recent {\em Large Language Model} (LLM)-driven approaches~\cite{fan2024cost, fu2025context, wang2025match}.
Specifically, learning-based methods span from unsupervised probabilistic generative models like \texttt{ZeroER}~\cite{wu2020zeroer} to advanced deep learning architectures, including {\em graph neural networks} (GNNs)~\cite{yao2022entity} and {\em pre-trained language models} (PLMs)~\cite{ge2021collaborem,li2020deep},
which frame ER as a sequence pairwise classification task. Despite being effective, this methodology typically demands extensive labeled data or substantial training overhead.
Crowdsourcing-based strategies are designed to handle highly ambiguous cases when automated methods fail. By regarding the ER as a number of {\em Human Intelligence Tasks} (HITs), these methods pioneer the use of active learning and budget-constrained optimization to select the most informative pairs or sets for human verification~\cite{wang2012crowder,demartini2012zencrowd}, guided by a philosophy of cost-effectiveness. However, relying on human workers is intrinsically slow, costly, and difficult to scale. With the advent of LLMs, LLM-driven approaches have opened a new frontier for ER by tapping into the extensive knowledge, semantic understanding, and human-like reasoning capabilities of LLMs.
\citet{narayan2022can} first leverage LLMs to perform binary classification on record pairs, inspiring a series of follow-up studies~\cite{fan2024cost, li2024leveraging, zhang2023large}. Amid them, state-of-the-art approaches typically leverage LLMs for precise pairwise verification~\cite{fan2024cost,zhang2023large} like \texttt{BatchER}~\cite{fan2024cost} or set-wise selection and clustering~\cite{wang2025match,fu2025context} as in \texttt{LLM-CER}~\cite{fu2025context}. 

Despite recent advancements, the majority of existing solutions, including the recent LLM-based ones, predominantly adhere to the canonical {\em blocking-matching-clustering} (BMC) workflow~\cite{christophides2020overview, papadakis2020blocking, elmagarmid2006duplicate}. This cascaded pipeline first partitions records into blocks to reduce complexity, then identifies matches, and finally groups them into clusters.
However, the decoupled nature of the BMC paradigm incurs critical deficiencies, particularly for {\em dirty} ER tasks. First, the static blocking phase creates a theoretical {\em recall ceiling}~\cite{papadakis2020blocking}. More precisely, blocking is typically a one-off filtering step based on surface-form similarities, and hence, true matches separated into different blocks are permanently lost. Neither advanced PLMs nor powerful LLMs can recover these lost pairs in subsequent stages, strictly bounding the maximum achievable recall~\cite{niknam2022role}. 
Second, the decoupled workflow prevents mutual refinement between stages. 
The matching phase often performs ``blind matching'' without awareness of the global graph structure (e.g., {\em transitivity}~\cite{wang2013leveraging, niknam2022role, bansal2004correlation}), while the clustering phase is forced to operate on a static graph plagued by missing edges and conflicting links caused by upstream errors. Third, regarding cost-efficiency, existing LLM-based methods~\cite{fan2024cost,zhang2023large,wang2025match,fu2025context} lack the sophistication of crowdsourcing strategies. They often rely on exhaustive queries or simple heuristics, failing to adaptively distinguish between ``easy'' cases resolvable by transitivity and ``hard'' tasks that truly necessitate expensive LLM reasoning, leading to wasteful budget consumption.

In response to these challenges, we present \algo{} (\underline{A}daptive \underline{L}abel \underline{P}ropagation for \underline{E}ntity \underline{R}esolution), a holistic framework that fundamentally redefines dirty ER as a dynamic optimization problem on an evolving graph.
Distinct from the rigid BMC paradigm that suffers from an inherent recall ceiling due to static blocking and blind matching, \algo{} unifies these disjoint stages into an iterative graph and label refinement process. It jointly optimizes graph topology and record labels, enabling the recovery of missing links through {\em transitive closure} that were permanently lost in traditional pipelines. More specifically, to balance precision and budget, \algo{} synergizes ``weak but free'' signals from structural transitivity with ``strong but expensive'' signals from LLMs.
Instead of exhaustive querying, we employ \emph{probabilistic label propagation} to utilize neighborhood consistency for inferring ``easy'' matches, reserving the precious budget for ``hard'' boundary cases.
To rigorously control costs, \algo{} formulates the signal selection as an \emph{Online Knapsack Problem}~\cite{chakrabarty2008online}. We introduce an uncertainty-based metric, \textit{marginal value gain} (MVG), derived from {\em label perplexity}, to decide when to trigger pairwise queries to LLMs dynamically.
This theoretically-grounded strategy ensures that the budget is adaptively allocated to queries that offer the maximum information gain for global graph structure refinement.
Furthermore, \algo{} incorporates a \emph{local update} mechanism where LLMs act not just as verifiers but as reasoners to select the best match from top-$k$ candidates, effectively repairing the graph topology and mitigating error propagation. Our extensive experiments on eight real-world benchmark datasets demonstrate that \algo{} significantly outperforms state-of-the-art baselines (including \texttt{BatchER} and \texttt{LLM-CER}) in terms of both ER performance and cost-effectiveness. The consistent superiority of \algo{}, particularly on complex datasets with high entity dispersion, highlights the effectiveness of our iterative graph refinement framework and adaptive signal selection technique.

%% file: tex/relatedwork.tex
\section{Related Work}

We broadly categorize existing ER literature into learning-based, crowdsourcing-based, and LLM-driven paradigms. Due to space constraints, we defer detailed discussions on traditional rule-based methods~\cite{fellegi1969theory,yujian2007normalized}, early GNN-based models~\cite{konda2016magellan, wang2017efficient}, and specific blocking techniques~\cite{paulsen2023sparkly,ebraheem2017deeper} to Appendix~\ref{app:related_work}.

\subsection{Learning-based Entity Resolution}

Recent advancements in learning-based ER have shifted from hand-crafted features to deep learning, particularly PLMs. Supervised methods, such as \texttt{Ditto}~\cite{li2020deep} and \texttt{JointBERT}~\cite{peeters2021dual}, fine-tune PLMs to capture fine-grained semantic interactions in pairwise matching. To circumvent the high cost of annotation, unsupervised and self-supervised frameworks like \texttt{ZeroER}~\cite{wu2020zeroer} and \texttt{CollaborEM}~\cite{ge2021collaborem} have emerged, utilizing generative models or pseudo-labeling strategies to learn from data without explicit human labels. Notably, approaches like \texttt{CollaborEM} attempt to bridge the gap between structural and semantic signals by integrating GNNs with PLMs. However, these methods typically require extensive labeled samples or incur prohibitive training overheads, limiting their adaptability in low-resource scenarios.

\subsection{Crowdsourcing-based Entity Resolution}
\label{sec:rel_crowd}

Pioneered by \texttt{CrowdER}~\cite{wang2012crowder} and \texttt{ZenCrowd}~\cite{demartini2012zencrowd}, crowdsourcing approaches fundamentally reframe ER as a budget-constrained optimization problem. These methods treat ER as a series of HITs, balancing resolution accuracy against monetary cost. To minimize expensive human annotation, varying optimization strategies are developed: active learning frameworks~\cite{mozafari2014scaling, meduri2020comprehensive} and decision-theoretic models~\cite{yalavarthi2017select} select only the most informative pairs for verification, while partial-order approaches~\cite{chai2016cost} prune the search space by inferring remaining relations. Furthermore, transitivity-based inference~\cite{vesdapunt2014crowdsourcing, wang2013leveraging} is introduced to logically deduce links and skip redundant queries. Despite these theoretical innovations, crowdsourcing faces intrinsic scalability bottlenecks due to the high cost and latency of human labor and the need for redundant verification to mitigate non-expert worker errors.

\subsection{LLM-based Entity Resolution}

In the generative AI era, LLMs offer a scalable and knowledgeable ``pseudo-crowd''~\cite{narayan2022can}. Early works utilize LLMs for binary classification on record pairs~\cite{narayan2022can}. To mitigate high inference costs, approaches like \texttt{BatchER}~\cite{fan2024cost} introduce batch prompting to amortize costs across multiple comparisons. However, these methods often treat comparisons as independent local decisions, overlooking global consistency. Recent state-of-the-art methods leverage LLMs for holistic reasoning: \texttt{ComEM}~\cite{wang2025match} employs a ``match-compare-select'' interaction to resolve ambiguities, while \texttt{LLM-CER}~\cite{fu2025context} utilizes in-context clustering to partition record sets directly. 
Despite these advancements, predominant LLM-based methods strictly adhere to the canonical BMC workflow. They operate solely on candidate blocks generated by static blocking techniques (e.g., \texttt{Sparkly}~\cite{paulsen2023sparkly} or LSH~\cite{ebraheem2017deeper}). Consequently, they suffer from an inherent \textit{recall ceiling}~\cite{papadakis2020blocking, niknam2022role}: true matches separated during the one-off blocking phase are permanently lost, as the downstream LLM has no opportunity to recover them.

%% file: tex/preliminary.tex
\section{Preliminaries}

\subsection{Problem Statement}

Let $\mathcal{R} = \{r_1, r_2, \dots, r_n\}$ be a collection of $n$ records, where each record is associated with a set of (textual, numerical, or categorical) attributes describing a real-world entity. We call $r_i$ and $r_j$ {\em duplicates}, denoted by $r_i\equiv r_j$, if they refer to the same entity.
For known duplicates, they obey the following transitive relations~\cite{benjelloun2009swoosh,wang2013leveraging}.
\begin{itemize}[leftmargin=*]
\item {\em Positive Transitivity}: if $r_1\equiv r_2$ and $r_2\equiv r_3$, then $r_1\equiv r_3$.
\item {\em Negative Transitivity}: if $r_1\equiv r_2$, but $r_2\not\equiv r_3$, then $r_1\not\equiv r_3$.
\end{itemize}

The goal of {\em dirty entity resolution} (dirty ER, a.k.a. deduplication)~\cite{christen2012data} is to partition $\mathcal{R}$ into a set of disjoint clusters (sets of duplicate entity profiles) $\mathcal{C} = \{\C_1, \C_2, \dots, \C_{|\C|}\}$, such that each cluster $\C_k$ corresponds to a unique real-world entity, i.e., $\forall{r_i,r_j\in C_k}\ r_i\equiv r_j$ and $\forall{r_i,r_j\in \C_k\times \C_\ell}$ with $k\neq\ell$, $r_i\not\equiv r_j$.
Particularly, these clusters $\{\C_1, \C_2, \dots, \C_{|\C|}\}$ are said to be transitively-closed.

\subsection{Blocking-Matching-Clustering Paradigm}\label{sec:BMC}
The state-of-the-art ER systems, including deep learning-based and LLM-based approaches, predominantly adopt a \textit{blocking-matching-combining} (BMC) paradigm to handle the quadratic complexity of the problem~\cite{thirumuruganathan2021deep, fan2024cost, fu2025context}. 
Given the set $\mathcal{R}$ of input records, the BMC workflow operates as follows:
\begin{enumerate}[leftmargin=*]
\item \textbf{Blocking:} Potentially similar records in $\mathcal{R}$ are grouped into (usually overlapping) buckets (``blocks'') such that each block $\mathcal{B}_k \subset \mathcal{R} \times \mathcal{R}$ and $|\mathcal{B}_k| \ll n^2$. 
The goal is to generate manageable subsets of candidates to limit comparisons in the next stage. 
\item \textbf{Matching:} A sophisticated matcher $\Phi$ (e.g., a cross-encoder, crowd workers, or an LLM) classifies candidate pairs within each block $\mathcal{B}_k$ as matches or non-matches, i.e., via {\em pairwise comparisons/queries}.
\item \textbf{Clustering:} The clustering (or combining) phase groups pairwise matches into disjoint clusters by enforcing the aforementioned transitive relations.
\end{enumerate}

\begin{figure}[!t]
    \centering
    \includegraphics[width=0.95\columnwidth]{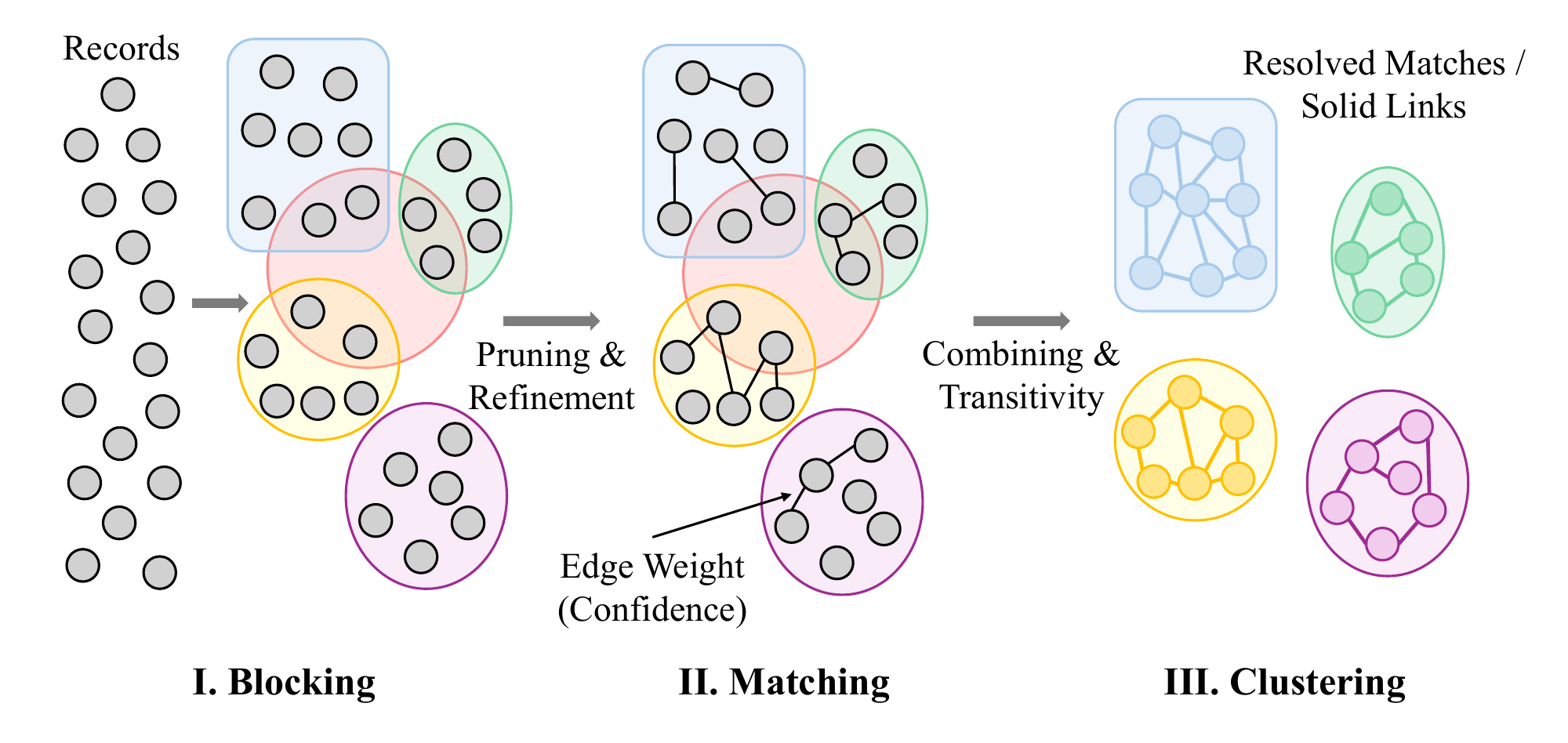} 
    \vspace{-3ex}
    \caption{Graph interpretation of the BMC paradigm.}
    \label{fig:blocking-matching}
\vspace{-3ex}
\end{figure}

\stitle{A Graph-based Interpretation} 
By regarding each record as a node and the positive match of two records as an edge, the ground-truth entity clusters $\mathcal{C}$ are essentially a list of {\em isolated cliques}~\cite{vadrevu2021xer}.
The BMC paradigm can thus be interpreted from a graph perspective. More concretely, as illustrated in Fig.~\ref{fig:blocking-matching}, the blocking stage acts as a static filter for edge pruning. The blocking initializes a list of {\em empty} candidate graphs (i.e., blocks).
Subsequently, the matching phase refines each candidate graph separately by establishing edges for record pairs with high confidence, yielding (weighted) sparse graphs. 
Finally, the above refined candidate graphs are viewed as an entire graph $\G$ containing all detected pairwise matches as edges. The combining step seeks to derive connected components from $\G$ via {\em transitive closure}~\cite{hassanzadeh2009framework,saeedi2018scalable}, or to create clusters that maximize agreement (solid links within a cluster)~\cite{vesdapunt2014crowdsourcing,bansal2004correlation}, thereby enforcing transitive constraints and resolving inconsistencies.

\subsection{Limitations of the BMC Paradigm}\label{sec:limits_phased}

In what follows, we pinpoint that the cascaded pipeline of the BMC paradigm is suboptimal, as its three stages are strictly decoupled, making the performance in the latter phase heavily rely on or be limited by the result quality of the former stages.

\stitle{Static Blocking}
Firstly, although the blocking stage significantly reduces the amount of candidate record pairs for comparisons or queries in the matching stage (from $O(N^2)$ to a superlinear but subquadratic time complexity~\cite{papadakis2020blocking}), this stage builds static and fixed blocks, leading to severe \textit{recall ceiling}~\cite{papadakis2020blocking}.
That is, considerable true matches fall into distinct blocks after blocking, which are hard or even unable to recover in subsequent stages~\cite{niknam2022role}.
As revealed in ~\cite{papadakis2020blocking}, this theoretical bottleneck is caused by the fact that the blocking phase is to attain a trade-off among three conflicting metrics for efficiency, precision/cost-control, and recall.

To illustrate, we empirically evaluate the \textit{ratio of cross-block matches} (CBMR), i.e., the proportion of true matches whose two records exist in distinct blocks, by various blocking strategies, including \texttt{LSH}~\cite{ebraheem2017deeper}, \texttt{Canopy}~\cite{mccallum2000efficient}, \texttt{Sparkly}~\cite{paulsen2023sparkly}, and \texttt{$K$NN}~\cite{thirumuruganathan2021deep}. 
Fig.~\ref{fig:cbmr_analysis} reports the CBMR values of these methods over real datasets {\it Alaska} and {\it Movies} when varying the {\em redundancy factor}, i.e., the average number of blocks a record is assigned to.
Particularly, the redundancy factor is commonly set to a constant between 10 and 20 for balancing efficiency and effectiveness in practice~\cite{fu2025context,fan2024cost,wang2025match}.
It can be observed that all blocking approaches still yield high CBMRs under such settings, especially on {\em Movies}.

\begin{figure}[!t]
\centering
\begin{small}
\begin{tikzpicture}
    \begin{customlegend}
    [legend columns=4, %
        legend entries={\texttt{LSH}, \texttt{$K$NN}, \texttt{Canopy}, \texttt{Sparkly}},
        legend style={at={(0.5,0)},anchor=north,draw=none,font=\footnotesize,column sep=0.3cm}]
    
    \addlegendimage{line width=0.6mm,mark=none,mark size=2.5pt,color=myblue2, dotted}
    \addlegendimage{line width=0.6mm,mark=none,mark size=2.5pt,color=teal}
    \addlegendimage{line width=0.6mm,mark=none,mark size=2.5pt,color=cyan}
    \addlegendimage{line width=0.6mm,mark=none,mark size=2.5pt,color=green!60!black, dashed}
    \end{customlegend}
\end{tikzpicture}
\\[-\lineskip]
\vspace{-3ex}
\subfloat[\em Alaska (Moderate)]{
\begin{tikzpicture}[scale=1,every mark/.append style={mark size=2.5pt}]
    \begin{axis}[
        height=\columnwidth/2.8, %
        width=\columnwidth/2.1,  %
        ylabel={\it CBMR},       %
        xmin=1, xmax=32,
        xtick={1,10,20,30},
        xticklabel style = {font=\footnotesize},
        yticklabel style = {font=\footnotesize},
        ymin=0.2, ymax=1.05,     %
        ytick={0.2, 0.4, 0.6, 0.8, 1.0},
        every axis y label/.style={font=\footnotesize,at={(current axis.north west)},right=10mm,above=0mm}, %
        legend style={draw=none}, %
    ]
    
    \addplot[line width=0.6mm, mark=none, color=myblue2, dotted] 
        plot coordinates { (4.5, 0.86) (7.2, 0.80) (22.0, 0.68) (32.0, 0.62) };

    \addplot[line width=0.6mm, mark=none, color=teal] 
        plot coordinates { (1, 0.98) (10, 0.63) (20, 0.48) (30, 0.36) };

    \addplot[line width=0.6mm, mark=none, color=cyan] 
        plot coordinates { (5.0, 0.82) (10.5, 0.69) (20.5, 0.49) (32.0, 0.34) };
    
    \addplot[line width=0.6mm, mark=none, color=green!60!black, dashed] 
        plot coordinates { (1, 0.98) (10, 0.65) (20, 0.40) (30, 0.28) };

    \end{axis}
\end{tikzpicture}\label{fig:res-alaska}%
}
\hspace{2mm} %
\subfloat[\em Movies (Hard)]{
\begin{tikzpicture}[scale=1,every mark/.append style={mark size=2.5pt}]
    \begin{axis}[
        height=\columnwidth/2.8,
        width=\columnwidth/2.1,
        ylabel={\it CBMR}, 
        xmin=1, xmax=32,
        xtick={1,10,20,30},
        xticklabel style = {font=\footnotesize},
        yticklabel style = {font=\footnotesize},
        ymin=0.6, ymax=1.0, %
        ytick={0.6, 0.7, 0.8, 0.9, 1.0},
        yticklabels={0.6, 0.7, 0.8, 0.9, 1.0},
        every axis y label/.style={font=\footnotesize,at={(current axis.north west)},right=10mm,above=0mm},
    ]
    
    \addplot[line width=0.6mm, mark=none, color=myblue2, dotted] 
        plot coordinates { (2.0, 0.96) (12.0, 0.92) (25.0, 0.85) (32.0, 0.80) };

    \addplot[line width=0.6mm, mark=none, color=teal] 
        plot coordinates { (1, 0.97) (10, 0.85) (20, 0.75) (30, 0.68) };

    \addplot[line width=0.6mm, mark=none, color=cyan] 
        plot coordinates { (5.0, 0.95) (15.0, 0.88) (25.0, 0.78) (29.0, 0.72) };

    \addplot[line width=0.6mm, mark=none, color=green!60!black, dashed] 
        plot coordinates { (1, 0.97) (10, 0.80) (20, 0.72) (30, 0.65) };

    \end{axis}
\end{tikzpicture}\label{fig:res-movies}%
}
\end{small}
\vspace{-2ex}
\caption{The CBMR when varying redundancy factor. 
} 
\label{fig:cbmr_analysis}
\vspace{-3ex}
\end{figure}
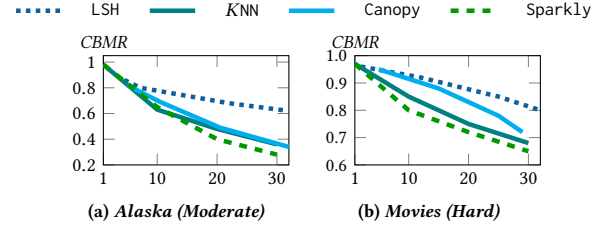

\stitle{Blind Matching} 
The matcher is locally focused, which classifies record pairs in each static block in isolation, often regardless of transitive constraints and the global graph structures.
Thus, the resulting graphs are almost always logically inconsistent, i.e., violating transitivity rules~\cite{bansal2004correlation, wang2015crowd}, rendering them hard to be clustered.
Although skilled workers or powerful LLMs can assist in enhancing the quality of edges/matches, which is widely adopted in the literature~\cite{wang2012crowder, narayan2022can}, the required monetary and computational costs are still significant~\cite{whang2013question, fan2024cost, wang2025match}.

\stitle{Clustering on Corrupted Graphs} 
As analyzed above, the blocking and matching steps actually create a static and sparse graph consisting of considerable missing and conflicting edges.
The clustering over such graphs can easily be misled by the conflicts and missing links, particularly by following the transitivity rigidly, which is known as {\em error propagation}~\cite{wang2015crowd}.

%% file: tex/method.tex
\section{Methodology}
\label{sec:methodology}
This section presents \algo{}, a graph-based framework that iteratively and adaptively refines the topology and cluster labels for dirty ER under the budget constraint. 
We begin with the high-level idea of \algo{} in \S\ref{sec:idea} and an overview of its iterative workflow in \S\ref{sec:overview}, followed by 
introducing our signal selection scheme in \S\ref{sec:signal-select}.
The secondary algorithm for local graph refinement is elaborated in \S\ref{sec:local-update}.
Lastly, \S\ref{sec:analysis} offers a detailed analysis of \algo{} in terms of computational complexity and query cost.

\subsection{\bf High-level Idea}\label{sec:idea}

Our analysis in \S\ref{sec:limits_phased} reveals that the cascaded BMC workflow constructs incomplete and noisy graphs for clustering due to its static blocking and blind matching.
Ideally, a global graph over records that allows dynamic updates and refinements is needed, instead of simply partitioning records into disjoint fixed blocks.

Further, we pinpoint that the decoupled matching and clustering phases actually optimize the same objective using complementary signals from distinct sources.
To be specific, the matching refines the edges (graph topology) via pairwise queries to a matcher, whilst the clustering step infers missing connections and resolves inconsistencies by enforcing transitive consistency (the graph-based logic) over the detected edges.
Both of them maximize the amount of true matches (positive edges) within entity clusters and the amount of true non-matches (non-edges) across entity clusters~\cite{vesdapunt2014crowdsourcing,bansal2004correlation}, which formally optimizes
\begin{small}
\begin{equation}\label{eq:obj-cluster}
\max_{\delta}\sum_{r_i,r_j\in \mathcal{R}}{w^{+}_{i,j}\cdot \delta_{i,j}} + \lambda\sum_{r_i,r_j\in \mathcal{R}}{w^{-}_{i,j}\cdot (1-\delta_{i,j})}.
\end{equation}
\end{small}
$w^{+}_{i,j}=1$ (resp. $w^{-}_{i,j}=1$) if $(r_i,r_j)$ is (resp. not) a true match and $\delta_{i,j}=1$ if $r_i,r_j$ are assigned to the same cluster (connected via an edge), and $0$ otherwise. 
As remarked earlier in \S\ref{sec:BMC}, the ideal graph is a list of cliques, each of which is an entity cluster.
As such, the identification of edges/matches (graph refinement) is mathematically equivalent to refining the cluster labels of the records.

Notably, these processes can mutually reinforce one another. By exploiting graph topology (i.e., transitive relations), we can drastically reduce the number of expensive pairwise queries required by the matcher~\cite{benjelloun2009swoosh}. On the other hand, high-confidence links by the matcher enhance the graph quality and hence mitigate error propagation.
This motivates an adaptive framework that unifies the matching and clustering, i.e., graph and label refinement. 

\begin{figure}[!t]
    \centering
    \includegraphics[width=0.95\columnwidth]{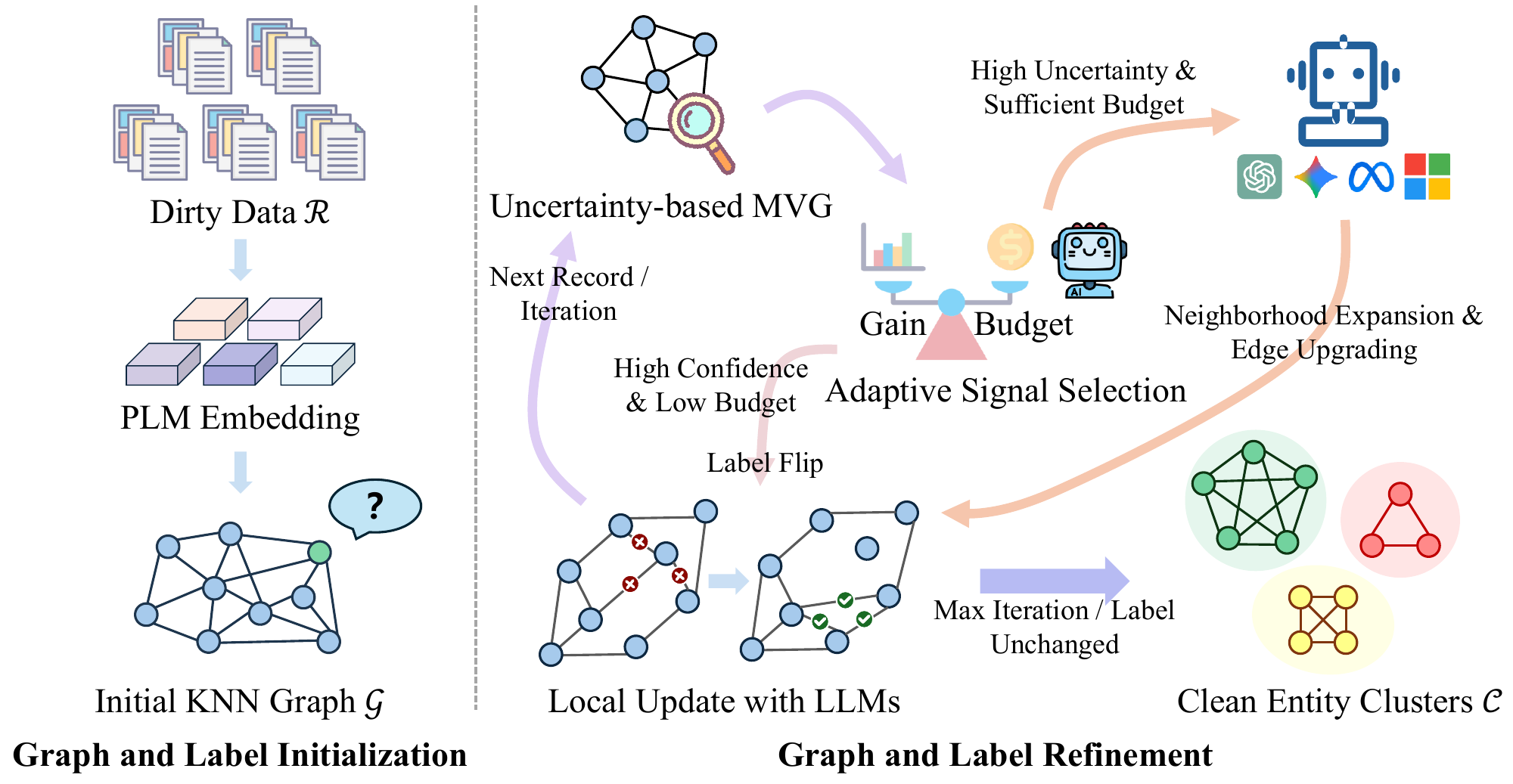} 
    \vspace{-3ex}
    \caption{An overview of \algo{}}
    \label{fig:framework}
\vspace{-2ex}
\end{figure}

\begin{lemma}\label{lem:LP}
Given a graph $\G$ of records, the {label propagation}~\cite{raghavan2007near} over $\G$ is equivalent to optimize Eq.~\eqref{eq:obj-cluster} with $\lambda=0$.
\end{lemma}

Inspired by the foregoing insights and Lemma~\ref{lem:LP}, our idea is to establish the \algo{} framework upon the graph-based {\em label propagation}~\cite{raghavan2007near}, wherein the graph structures and node labels are iteratively and jointly refined or optimized with complementary signals from the LLM-based matcher and graph neighborhoods.
Particularly, the weighted propagation of labels via neighbors enables the use of {\em soft transitivity} to infer ``easy'' matches probabilistically without query cost, which not only attenuates the large errors amplified by applying rigid clustering transitivity over imperfect graphs but also reserves the budget of expensive LLM queries for ``hard'' cases to correct the graph and labels.
Fig.~\ref{fig:framework} provides a figurative illustration of the idea underlying our \algo{} framework.

\subsection{Framework Overview}\label{sec:overview}

\algo{} consists of two stages, where the former focuses on initializing the graph and node labels, while the latter iteratively and adaptively refines the graph and labels, as outlined in Algo.~\ref{alg:ce_ppl}.

\stitle{Graph and Label Initialization} After taking as input the record set $\mathcal{R}$, total (monetary) budget $B$ for querying LLMs, the maximum number $T_\text{max}$ of iterations, coefficients $\alpha$ and $\theta$, \algo{} starts by assigning a unique label $\ell(r_i)$ to each record $r_i$ in $\mathcal{R}$ and generating its embedding vector $\rvec_i$ using a pre-trained language model (PLM) (Line 1).
Next, a $K$-nearest neighbor ($K$NN) graph $\G$ is approximately constructed~\cite{thirumuruganathan2021deep} over the embedding vectors $\{\rvec_i\}_{i=1}^n$ (Line 2), such that each record $r_i$ is associated with $K$ neighbors $\N(r_i)$.
Accordingly, the weight $w(r_i,r_j)$ of each edge $(r_i,r_j)$ in $\G$ is initialized as follows:
\begin{equation}\label{eq:weight-init}
w(r_i,r_j)\gets \alpha \cdot\textsf{cos}(\rvec_i,\rvec_j),
\end{equation}
i.e., the $\alpha$-scaled ($0<\alpha\le 1$) cosine similarity of the embedding vectors of $r_i$ and $r_j$ (Line 3). 
In the meantime, the consumed budget $\beta$ is initially set to be $0$ at Line 4. 

\begin{algorithm}[!t]
\footnotesize
\caption{The \algo{} framework}\label{alg:ce_ppl}
\KwIn{Record set $\mathcal{R}$, budget $B$, integer $T_\text{max}$, coefficients $\alpha$ and $\theta$}
\KwOut{Entity clusters $\mathcal{C} = \{\C_1, \C_2, \dots, \C_{|\C|}\}$.}
\Comment{Graph and Label Initialization}
$\ell(r_i)\gets i;\ \rvec_i\gets \textsf{PLM}(r_i)\ \forall{r_i\in \mathcal{R}}$\;
$\{\N(r_i)\}_{i=1}^n\gets K\text{NN}(\{\rvec_i\}_{i=1}^n)$\;
Initialize edge weights $w(r_i,r_j)\ \forall{(r_i,r_j)\in \G}$\tcp*{Eq.~\eqref{eq:weight-init}}
Initial consumed budget $\beta \gets 0$\;
\Comment{Graph and Label Refinement}
\Repeat{$t \ge T_\textnormal{max}$ \textnormal{\textbf{or}} $\{\ell(r_i)\}_{i=1}^n$ remain unchanged}{
    \For{$r_i \in \mathcal{R}$}{
        Update label distribution $\boldsymbol{\pi}_i$ \tcp*{Eq.~\eqref{eq:label-dist}}
        \If{\textnormal{Algo.~\ref{alg:llm-select} returns} {true}}{
            Invoke Algo.~\ref{alg:local-refinement} for $r_i$ and let $\omega_i$ be the actual cost\;
            $\beta \gets \beta + \omega_i$\;
        }
        \Else{
            $\iota^* \gets \argmax{\iota\in \mathcal{L}_i}{\boldsymbol{\pi}_i(\iota)}$\;
            \lIf{$\boldsymbol{\pi}_i(\iota^*) > \theta$}{ 
                $\ell(r_i) \gets \iota^*$;
            }
        }
    }
    $t \gets t + 1$\;
}
Convert $\{\ell(r_i)\}_{i=1}^n$ into entity clusters $\mathcal{C}$\;
\end{algorithm}

\stitle{Iterative Graph and Label Refinement} Afterwards, \algo{} proceeds to an iterative process to jointly refine $\G$ and $\{\ell(r_i)|\forall{r_i\in \mathcal{R}}\}$ based on the graph-based label propagation and pairwise queries to LLMs (Lines 5-16).
More specifically, in each iteration, for each record $r_i\in \mathcal{R}$, \algo{} first updates the label distribution $\boldsymbol{\pi}_i$ of $r_i$ w.r.t. its candidate labels $\mathcal{L}_i=\{\ell(r_j)|r_j\in \N(r_i)\}$ formed by its neighbors $\N(g_i)$ in $\G$ (Line 7). In mathematical terms, the probability of assigning label $\iota$ to record $r_i$ is calculated by 
\begin{small}
\begin{equation}\label{eq:label-dist}
    \boldsymbol{\pi}_i(\iota) = \frac{1}{\sum_{r_j \in \mathcal{N}(r_i)}{w(r_i,r_j)}}\cdot \sum_{r_j \in \mathcal{N}(r_i)} w(r_i,r_j) \cdot \mathbb{1}_{\ell(r_j)=\iota}
\end{equation}
\end{small}
where $\mathbb{1}_{\ell(r_j)=\iota}$ is $1$ if $\ell(r_j)=\iota$ and $0$ otherwise. 

Based on the evaluation in Algo.~\ref{alg:llm-select} with the label distribution $\boldsymbol{\pi}_i$ and available query budget, \algo{} then adaptively determines the use of {\em weak and free signals} (from label propagation over $\G$) or {\em strong but expensive signals} (from pairwise queries to LLMs) for the refinement of $r_i$'s label $\ell(r_i)$ and its neighborhood $\N(r_i)$ in $\G$. 
This design prioritizes the allocation of query budget to records at decision boundaries, while simple tasks can be cheaply done via weighted label propagation.
Particularly, the subsequent operation is either one of the following two:
\begin{itemize}[leftmargin=*]
\item \textbf{Local Update with LLMs} (Lines 8-10): if Algo.~\ref{alg:llm-select} returns a positive answer, i.e., it is affordable and cost-effective to issue the pairwise queries for $r_i$, \algo{} will query the LLM, and update $r_i$'s label and neighbors based on its response accordingly by invoking the \texttt{LocalUpdate} algorithm (Algo.~\ref{alg:local-refinement});
\item \textbf{Weighted Label Propagation} (Lines 11-13): if Algo.~\ref{alg:llm-select} returns false, implying that $r_i$ already has a confident label or the remaining budget $B-\beta$ is insufficient/depleted, \algo{} simply takes the most confident label $\iota^\ast$ among $r_i$'s neighbors as its new label if the confidence $\boldsymbol{\pi}_i(\iota^\ast)$ is beyond the threshold $\theta$.
\end{itemize}
The iterative refinement repeats for all records until convergence (i.e., the labels of all records remain invariant) or maximum iterations are reached.
Finally, the eventual labels are then converted into a set of entity clusters at Line 16.

\subsection{Adaptive Signal Selection}\label{sec:signal-select}

\stitle{Optimization Objective}
For each target record $r_i$ at the  $t$-th iteration, the goal is to make an immediate and irrevocable decision on whether to invest a small budget $\omega_i$ out of $B$ to call LLMs for strong signals or directly leverage cost-free graph signals for updating $r_i$'s neighborhood and label. Particularly, the iterative refinement for all records can be perceived as a dynamic resource allocation, which involves a sequence of binary decision makings.
This optimization process can be framed as an \textit{Online Knapsack Problem}~\cite{chakrabarty2008online}, where the knapsack capacity is the total budget $B$.
More precisely, the optimization objective is formulated as follows:
\begin{small}
\begin{equation}
    \max_{\chi} \sum_{t=1}^{T_\textnormal{max}} \sum_{r_i \in \mathcal{R}} \chi_{i,t} \cdot g_i \quad \text{s.t.} \sum_{t=1}^{T_\textnormal{max}} \sum_{r_i \in \mathcal{R}} \chi_{i,t} \cdot \omega_i \le B,
\label{eq:okp_obj}
\end{equation}
\end{small}
where $\chi_{i,t} \in \{0, 1\}$ is the binary decision variable, indicating an invocation of the LLM for $r_i$ at the $t$-th iteration if $\chi_{i,t}=1$, and $g_i$ stands for the \textit{marginal value gain} (MVG) for record $r_i$. 
In brief, the overall optimization goal is to maximize the cumulative marginal gain subject to the budget constraint $B$.

\begin{algorithm}[!t]
\footnotesize
\caption{Adaptive Signal Selection}\label{alg:llm-select}
\KwIn{Target record $r_i$, total budget $B$, consumed budget $\beta$}

Calculate $\Delta^\text{\scriptsize{WLP}}_i$ \tcp*{Eq.~\eqref{eq:graph_confidence}}

Estimate the cost $\omega_i$ of querying the LLM with $r_i$\;

Calculate the marginal value gain $g_i$\tcp*{Eq.~\eqref{eq:MVG}}

\Return \textbf{true} if Inequality~\eqref{eq:use-LLM} holds and \textbf{false} otherwise\;

\end{algorithm}

\stitle{Uncertainty-based MVG}
We define the MVG $g_i$ as the expected reduction in uncertainty provided by the LLM over the weighted label propagation (WLP) in terms of updating $r_i$'s label:
\begin{equation}\label{eq:MVG}
    g_i \gets \max\left(0, \Delta^\text{\scriptsize{LLM}}_i - \Delta^\text{\scriptsize{WLP}}_i\right),
\end{equation}
where $\Delta^\text{\scriptsize{LLM}}_i$ and $\Delta^\text{\scriptsize{WLP}}_i$ symbolize the confidence/certainty for this update using two sources of signals, respectively.
When $g_i=0$, the WLP without incurring LLM query cost is deemed adequate or superior in inferring $r_i$'s new label, whereas $g_i > 0$ indicates that investing part of the budget for this positive gain {\em might} be a choice.

Next, we propose to formulate $\Delta^\text{\scriptsize{WLP}}_i$ based on the label distribution $\boldsymbol{\pi}_i$ of $r_i$ and the {\em label perplexity} in Definition~\ref{def:LP}.
\begin{definition}[Label Perplexity]\label{def:LP} Given a record $r_i$ and its label distribution $\boldsymbol{\pi}_i$, the {\em label perplexity} of $r_i$ is defined by $\rho(r_i) = \exp\left( - \sum_{\iota \in \mathcal{L}_i} \boldsymbol{\pi}_i(\iota)\cdot \ln \boldsymbol{\pi}_i(\iota) \right)$.
\end{definition}
Note that $\rho(r_i)\in [1,|\mathcal{L}_i|]$ is formulated upon Shannon entropy~\cite{shannon1948mathematical}, which can measure the overall certainty about $r_i$'s all possible labels. Intuitively, a high $\rho(r_i)$ connotes that the predicted probabilities in $\boldsymbol{\pi}_i$ are evenly distributed among all possible labels, making it hard to assign a label to $r_i$ confidently.
To quantify the confidence/certainty of assigning the best label based on $\boldsymbol{\pi}_i$ to $r_i$ in WLP (Lines 13-14 in Algo.~\ref{alg:ce_ppl}), \algo{} calculates $\Delta^\text{\scriptsize{WLP}}_i$ by
\begin{equation}
    \Delta^\text{\scriptsize{WLP}}_i \gets \left(1 - \log_{b_i}{\rho(r_i)}\right)\cdot \max_{\iota\in \mathcal{L}_i} \boldsymbol{\pi}_i(\iota),
\label{eq:graph_confidence}
\end{equation}
where $b_i = \max(|\mathcal{L}_i|, 2)$ is used for normalization, i.e., ensuring $0\le \log_{b_i}{\rho(r_i)}\le 1$.

Since $\Delta^\text{\scriptsize{LLM}}_i$ is dependent on the capacity of adopted LLMs, we set it to $0.95$ by default. Practically, a preliminary estimation of $\Delta^\text{\scriptsize{LLM}}_i$ can be done with a small test set, as in \cite{li2024leveraging, huang2025thriftllm}.

\stitle{Algorithm and Analysis}
Algo.~\ref{alg:llm-select} displays the pseudo-code for solving the optimization problem in Eq.~\eqref{eq:okp_obj}.
More concretely, Algo.~\ref{alg:llm-select} first calculates $\Delta^\text{\scriptsize{WLP}}_i$ according to Eq.~\eqref{eq:graph_confidence} at Line 1, and at Line 2 estimates the cost $\omega_i$ of querying the LLM for target record $r_i$ based on the number of tokens needed and the LLM API pricing using Eq.~\eqref{eq:API-cost}.
Accordingly, the MVG of $r_i$ is then obtained by Eq.~\eqref{eq:MVG} (Line 3).
At last, Algo.~\ref{alg:llm-select} returns true to indicate the selection of LLM signals for $r_i$, if the following inequalities are satisfied:
\begin{small}
\begin{equation}\label{eq:use-LLM}
\frac{g_i}{\omega_i} \ge \frac{L}{e} \left( \frac{U \cdot e}{L} \right)^{\frac{\beta}{B}}\ \text{and}\ \beta + \omega_i \le B,
\end{equation}
\end{small}
and false otherwise (Line 4). 
Note that $\frac{g_i}{\omega_i}$ denotes the marginal gain per unit cost, which is assumed to lie within a known range $[L, U]$.
The above conditions indicate that the LLM queries are allowed if the remaining budget is sufficient, and meanwhile, the per-cost marginal gain exceeds a particular threshold, which increases exponentially as the budget depletes and ensures that scarce resources are reserved for increasingly high-value LLM signals later.
This deterministic strategy is inspired by~\cite{chakrabarty2008online}, and offers a provably optimal competitive ratio as stated in the following theorem:

\begin{theorem}\label{lem:comp-ratio}
Algo.~\ref{alg:llm-select} achieves a competitive ratio $\ln(U/L) + 1$ for the total marginal gain.
\end{theorem}

\begin{algorithm}[!t]
\footnotesize
\caption{Local Update with LLMs}\label{alg:local-refinement}
\KwIn{Target record $r_i$}
\KwOut{The actual query cost $\omega_i$}
Pick the top-$m$ candidate labels $\mathcal{L}^{\ast}_i$ by $\boldsymbol{\pi}_i$\tcp*{Eq.~\eqref{eq:m-labels}}
$\mathcal{S}_i\gets \emptyset$\;
\For{$\iota\in \mathcal{L}^{\ast}_i$}{
Pick $r_i$'s nearest record $r_j$ in cluster $\C_{\iota}$\tcp*{Eq.~\eqref{eq:cluster-repre}}
$\mathcal{S}_i\gets \mathcal{S}_i\cup \{r_j\}$\;
}
$r\gets \textsf{LLM}(\mathcal{P}, r_i, \mathcal{S}_i)$ and compute the actual query cost $\omega_i$\;
\If{$r \neq \textnormal{none}$}{
        $\ell(r_i) \gets \ell(r)$\;
        \For{$r_j \in \C_{\ell(r)}$}{
            $\mathcal{N}(r_i) \leftarrow \mathcal{N}(r_i) \cup \{r_j\};\ \mathcal{N}(r_j) \leftarrow \mathcal{N}(r_j) \cup \{r_i\}$\;
            $w(r_i, r_j) \gets \sigma_\text{LLM}$\;
        }
                
            }
\lElse{$\forall{r_k \in \mathcal{S}_i}\ \N(r_i) \gets \N(r_i) \setminus \{r_{k}\}$; $\N(r_{k}) \gets \N(r_{k}) \setminus \{r_i\}$}

\end{algorithm}

\subsection{Local Update with LLMs}\label{sec:local-update}
Given a target record $r_i$, the module of local update with LLMs seeks to update $r_i$'s label and neighborhoods based on the response from the LLM.
The pseudo-code is illustrated in Algo.~\ref{alg:local-refinement}.
Firstly, Algo.~\ref{alg:local-refinement} focuses on selecting a set $\mathcal{S}_i$ of $m$ (typically $3$) candidate records with distinct labels for querying the LLMs at Lines 1-5.
To be specific, at Line 1, we identify the top-$m$ candidate labels from the label set $\mathcal{L}_i=\{\ell(r_j)|r_j\in \N(r_i)\}$ of $r_i$'s neighbors by
\begin{equation}\label{eq:m-labels}
\mathcal{L}^{\ast}_i \gets \underset{\iota\in \mathcal{L}_i}{{\arg\text{top-}m}}\ {\boldsymbol{\pi}_i(c)}.
\end{equation}
Based thereon, for each of the $m$ candidate labels $\iota$, we retrieve the corresponding cluster $\C_\iota$ that currently contains all records with label $\iota$, and find $r_i$'s nearest record $r_j$ in $\C_\iota$ (Line 3-4):
\begin{equation}\label{eq:cluster-repre}
r_j \gets \argmax{r_\ell\in \C_{\iota}}{\textsf{cos}(\rvec_i,\rvec_\ell)}
\end{equation}
Intuitively, $r_j$ can be interpreted as a representative of cluster $\C_\iota$, connoting the likelihood that $r_i$ joins $\C_\iota$.

Let $\mathcal{S}_i$ be the set of $m$ such representative records. Algo.~\ref{alg:local-refinement} then asks the LLM to choose a record $r$ from $\mathcal{S}_i$ such that $r_i$ and $r$ refers to the same entity: $r\gets\textsf{LLM}(\mathcal{P}, r_i, \mathcal{S}_i)$,
where $\mathcal{P}$ describes the task instruction for this multiple-choice question. Based thereon, we compute the actual cost $\omega_i$ for querying the LLM at Line 6. Due to space limits, we present the detailed prompt templates in Appendix~\ref{app:prompt}.
If the response $r$ by the LLM is a valid record, \algo{} considers $(r_i,r)$ as a positive match with a high confidence of $\sigma_\text{LLM}$, given the remarkable semantic understanding abilities of LLMs.
In turn, we update $r_i$'s label to $r$'s label at Line 8, and then proceed to refine its neighbors through symmetric {\em neighborhood expansion} and {\em edge upgrading}.
The former step expands $r_i$'s neighbors by establishing edges between $r_i$ and every record $r_j$ in $\C_{\ell(r)}$ (by positive transitivity), i.e., all the records with label $\ell(r)$ (Line 10), while the latter is to increase the weight of these links to a high value of $\sigma_\text{LLM}$ as they are directly or indirectly verified by LLMs (Line 11).
On the contrary, if the LLM returns an invalid $r$, which signifies that $r_i$ refers to an entity different from the records in $\mathcal{S}_i$, Algo.~\ref{alg:local-refinement} will remove their edges with $r_i$ (Line 12).

\subsection{Complexity Analyses}\label{sec:analysis}
We analyze the efficiency of \algo{} in terms of computation and monetary cost. Detailed derivations are provided in Appendix~\ref{app:complexity}.

\stitle{Computational Complexity}
The time complexity consists of graph initialization and iterative refinement. Initialization is dominated by the construction of the $K$NN graph, taking $O(n \log n)$. 
During refinement, \algo{} dynamically densifies the graph by neighborhood expansion (Algo.~\ref{alg:local-refinement}).
Let $\psi_{\text{max}}$ be the maximum size of the identified entity clusters during the refinement. 
The overall worst-case time complexity is thus bounded by $O(n \log n + T_\textnormal{max} \cdot n \cdot \psi_{\text{max}})$. 
Notice that although noise might slightly inflate clusters, $\psi_{\text{max}}$ remains comparable to the ground truth cluster size in practice, satisfying the \textit{sparse truth} assumption~\cite{betancourt2016flexible,draisbach2019transforming} where $\psi_{\text{max}} \ll N$.

\stitle{Monetary Cost}
Since \algo{} treats monetary cost as a hard constraint, the optimization in Algo.~\ref{alg:llm-select} ensures that the cumulative cost of LLM queries strictly satisfies $\sum \omega_i \le B$.

%% file: tex/experiments.tex
\section{Experiments}

This section experimentally evaluates the overall effectiveness and cost-effectiveness, the contribution of each component (e.g., adaptive signal selection, graph refinement), and the robustness of the proposed \algo framework. 
The source code and datasets are publicly accessible at \url{https://github.com/HKBU-LAGAS/Alper}. 

\subsection{Experimental Setup}
\input{figures/dataset}

Due to space constraints, we defer the detailed experimental setup, including dataset statistics, baseline descriptions, evaluation metrics, cost calculations, implementation details, and hyperparameter settings, to Appendix~\ref{sec:appendix_exp_setup}. Furthermore, the appendix presents extensive supplementary evaluations, such as hyperparameter sensitivity analyses (e.g., impact of blocking size $K$), 
and additional comparisons against supervised baselines.

\stitle{Datasets}
We conduct extensive experiments on eight diverse real-world dirty ER benchmarks~\cite{nikoletos2022pyjedai,papadakis2018return,saeedi2017comparative,crescenzi2021alaska}, as detailed in Table~\ref{tab:datasets}. These datasets span a broad spectrum of both scale, ranging from small ($n=841$ on \textit{Census}) to large ($n>50$k on \textit{Movies}), and structural topology, specifically, the entity dispersion ($E_d=n/|\C|$)~\cite{chen2005exploiting}, from low-dispersion scenarios like \textit{Amazon-GP} ($E_d \approx 1.34$) to high-density clusters on \textit{Cora} ($E_d \approx 11.56$).

\stitle{Baselines}
We benchmark \algo against five state-of-the-art baselines, categorized into two groups:
(1) \textit{Unsupervised and self-supervised approaches}, including \texttt{ZeroER}~\cite{wu2020zeroer} and \texttt{CollaborEM}~\cite{ge2021collaborem}, which do not rely on extensive human-labeled data; and
(2) \textit{LLM-based frameworks}, namely \texttt{BatchER}~\cite{fan2024cost}, \texttt{ComEM}~\cite{wang2025match}, and the state-of-the-art \texttt{LLM-CER}~\cite{fu2025context}.

\stitle{Metrics and Implementations}
To comprehensively assess the quality of the entity clusters, we employ two widely recognized metrics: \textit{FP-measure} (FP) and \textit{normalized mutual information} (NMI).
We also evaluate the \textit{monetary cost} incurred by each method.
All experiments are repeated three times, and we report the average values. To ensure a fair comparison, all methods use the same blocking method and parameters, and all LLM-based competitors utilize the same backbone model (e.g., \texttt{GPT-5-Mini}) and operate under identical monetary budget constraints as \algo.

\input{figures/effectiveness}

\input{figures/query-cost}

\subsection{Evaluation of ER Performance}

Table~\ref{tab:overall_performance_cost_effectiveness} presents a comprehensive comparison of \algo{} against state-of-the-art baselines in terms of both clustering quality (FP-measure, NMI) and resource consumption (token usage, API calls, and monetary cost).
We summarize the following key observations.

First, \algo{} consistently achieves the highest FP and NMI scores across all eight benchmarks, significantly outperforming both self-supervised methods (e.g., \texttt{CollaborEM}) and advanced LLM-based frameworks.
This advantage is most prominent on datasets with high entity dispersion or complex ambiguities. For instance, on \textit{Song} and \textit{Census}, \algo{} improves the FP-measure by substantial margins of $11.38\%$ and $10.59\%$ over the strongest competitor, \texttt{LLM-CER}.
Similarly, on the challenging \textit{Movies} dataset, \algo{} surpasses the runner-up by $6.05\%$ in FP and $8.40\%$ in NMI.
These results validate that \algo's core mechanism, synergizing transitive label propagation with adaptive LLM verification, successfully recovers missing links that static blocking methods (used by baselines) discard.

Second, a critical insight from Table~\ref{tab:overall_performance_cost_effectiveness} lies in the behavior of \texttt{LLM-CER}.
On datasets like \textit{Census}, \textit{Cora}, and \textit{Amazon-GP}, \texttt{LLM-CER} consumes less than the allocated budget (e.g., spending only $\$0.32$ out of $\$0.50$ on \textit{Census}).
This underutilization is due to the number of API calls for \texttt{LLM-CER} being strictly bounded by the initial static blocking and early merging decisions. Once these static candidates are exhausted, it essentially ``starves'', unable to utilize the remaining budget to find additional matches.
In contrast, \algo{} dynamically utilizes the available budget with graph structures, transforming potential unspent funds into tangible recall gains.

Another important observation is that despite exploring a larger search space, \algo{} demonstrates remarkable cost-effectiveness.
Compared to \texttt{BatchER}, which relies on exhaustive pairwise verification, \algo{} frequently requires fewer API calls to achieve higher accuracy (e.g., on \textit{Movies} and \textit{Song}), proving that our uncertainty-based signal selection effectively targets high-value queries.
Even when \algo{} incurs slightly higher costs than the ``starved'' \texttt{LLM-CER}, the return on investment is disproportionately high. For example, on \textit{AS}, a modest increase in token usage yields a massive $10.05\%$ improvement in NMI.
Moreover, on larger datasets such as \textit{Music} and \textit{Movies}, \algo{} achieves the best performance while incurring comparable token costs to \texttt{LLM-CER}.

\subsection{Performance when Varying Query Budgets}
To further investigate the scalability and robustness of \algo{}, we evaluate the FP-measure trends on \textit{Cora}, \textit{Alaska}, \textit{Music}, and \textit{Movies} as the query budget $B$ increases.
As illustrated in Figure~\ref{fig:query_results}, two distinct behaviors emerge, underscoring that \algo's superiority stems from \textit{dynamic structural discovery} rather than simply increasing query volume.
(1) The performance of baselines, particularly \texttt{BatchER} and \texttt{LLM-CER}, tends to plateau or exhibit diminishing returns even as the budget expands.
For instance, on \textit{Cora}, \texttt{LLM-CER} saturates at an FP of $\approx 84.5\%$ despite a budget increase from $0.3$ to $0.5$.
This confirms the existence of a \textit{recall ceiling}.
Meanwhile, \texttt{BatchER} necessitates a significantly larger budget for pair comparisons to attain comparable performance.
Simply increasing the budget fails to improve performance because these methods strictly operate within the closed set of pre-filtered candidates.
(2) In sharp contrast, \algo{} maintains a steep upward trajectory across all datasets.
On the \textit{Movies} dataset (Fig.~\ref{fig:res-movies}), \algo{} improves FP from $50.34\%$ to $64.68\%$ as the budget increases.
Crucially, this is not merely a result of ``buying'' more matches. Even under tight budgets (e.g., $B=0.4$ on \textit{Movies}), \algo{} already surpasses \texttt{BatchER} and \texttt{ComEM}, demonstrating superior query efficiency.
As the budget grows, \algo{} leverages its transitive label propagation to dynamically uncover \textit{new} candidate pairs that were initially blocked, effectively converting the additional budget into structural repairs that static baselines are architecturally incapable of performing.

\input{figures/ablation}

\input{figures/llms}

\subsection{Ablation Study and Component Analysis}
\label{sec:ablation}

Table~\ref{tab:merged-analysis} validates the contribution of \algo's key components and the robustness of its strategies.
First, regarding model components, removing either weighted label propagation (\texttt{w/o WLP}) or the active LLM module (\texttt{w/o LLM}) causes significant degradation. For instance, on the complex \textit{Movies} dataset, the full model outperforms both single-component variants by over 6\% in FP. This confirms that \algo succeeds by synergizing weak structural signals with strong semantic reasoning, breaking the cost-accuracy trade-off.
Second, for signal selection, our adaptive strategy consistently outperforms \texttt{Greedy} and \texttt{Random} strategies (e.g., +5.9\% FP over \texttt{Greedy} on \textit{Movies}). This indicates that blindly prioritizing high-uncertainty nodes (\texttt{Greedy}) wastes budget on hard but low-value cases, whereas our value-density-based approach effectively maximizes marginal gain.
Third, \algo proves robust to various graph initialization schemes (\texttt{LSH}, \texttt{Canopy}, \texttt{Sparkly}). While performance remains stable on simpler datasets like \textit{Cora}, our default semantic blocking (\texttt{SBERT+KNN}) provides a superior topological foundation for challenging tasks like \textit{Movies}, allowing the refinement budget to focus on subtle ambiguities. Finally, evaluating \algo with diverse LLMs, including \texttt{GPT-5-Mini}, \texttt{GPT-4o-Mini}, \texttt{Phi-4}, \texttt{DeepSeek-R1}, and \texttt{Llama-3.3-70B}, under a unified budget reveals consistent competitiveness (Fig.~\ref{fig:llm_comparison_new}). A distinct trade-off is observed wherein cheap models facilitate extensive querying at the risk of error propagation, whereas expensive models provide superior reasoning but rapidly deplete the budget, thereby constraining the graph refinement scope.

\subsection{Parameter Analysis}

We analyze the impact of three core hyperparameters: the edge weight scaling factor $\alpha$, the label propagation confidence threshold $\theta$, and the number of candidate matches $m$ for LLM verification.
As shown in Fig.~\ref{fig:hyperparameter_alpha}, varying $\alpha$ reveals that \algo performs consistently well across a broad range (e.g., $\alpha \in [0.6, 1.0]$). While performance dips slightly at lower values on \textit{Music}, it remains robust on \textit{Cora} and \textit{Alaska}.
Fig.~\ref{fig:hyperparameter_theta} indicates that the threshold $\theta$ for weak LP requires a balance; extremely low values may propagate noise, while values that are too high reduce effectiveness. However, the performance is generally stable around $\theta=0.6$, suggesting \algo is not overly sensitive to precise tuning.
Finally, regarding the candidate set size $m$ in Fig.~\ref{fig:hyperparameter_m}, we observe a performance gain as $m$ increases from 1 to 3, effectively plateauing thereafter.

\input{figures/parameters}

%% file: figures/dataset.tex
\begin{table}[!t]
\centering
\footnotesize
\caption{Dataset statistics.}
\label{tab:datasets}
\vspace{-3ex}
\renewcommand{\arraystretch}{0.95} 
\setlength{\tabcolsep}{4pt} 
\resizebox{0.9\columnwidth}{!}{
\begin{tabular}{l|c|c|c|c|c}
\hline
\textbf{Dataset} & \textbf{\#Records} & \textbf{\#Entities} & \textbf{\#Matches} & \textbf{Domain} & $\mathbf{E_d}$ \\ \hline
\textit{Census}    & 841    & 503    & 344     & Demo     & 1.67 \\
\textit{Cora}      & 1,295  & 112    & 17,184  & Citation & 11.56 \\
\textit{AS}        & 2,260  & 330    & 32,816  & Geo      & 6.85 \\
\textit{Amazon-GP} & 4,393  & 3,289  & 1,104   & Software & 1.34 \\
\textit{Song}      & 4,854  & 1,195  & 8,588   & Music    & 4.06 \\
\textit{Alaska}    & 12,007 & 1,484  & 144,440 & Product  & 8.09 \\
\textit{Music}     & 19,375 & 10,000 & 16,250  & Music    & 1.94 \\
\textit{Movies}    & 50,797 & 27,934 & 22,863  & Movie    & 1.82 \\
\hline
\end{tabular}
}
\vspace{-2ex}
\end{table}

%% file: figures/effectiveness.tex
\begin{table}[!t]
\centering
\footnotesize
\caption{Overall Performance (FP and NMI) and Cost Effectiveness. (best \textbf{bolded} and runner-up \underline{underlined})}
\label{tab:overall_performance_cost_effectiveness}
\vspace{-2ex}
\setlength{\tabcolsep}{2pt}
\renewcommand{\arraystretch}{0.9}
\resizebox{\linewidth}{!}{
\begin{tabular}{l|c|cc|cccc}
\toprule
\textbf{Dataset} & \textbf{Metric} & \rotatebox[origin=c]{20}{\footnotesize\texttt{ZeroER}} & \rotatebox[origin=c]{20}{\footnotesize\texttt{CollaborEM}} & \rotatebox[origin=c]{20}{\footnotesize\texttt{BatchER}} & \rotatebox[origin=c]{20}{\footnotesize\texttt{ComEM}} & \rotatebox[origin=c]{20}{\footnotesize\texttt{LLM-CER}} & \rotatebox[origin=c]{20}{\footnotesize\algo{}} \\
\midrule

\multirow{5}{*}{\shortstack[l]{\textit{Census} \\ \scriptsize ($B=\$0.5$)}} 
 & FP \textuparrow  & 49.43 & 55.86 & 62.77 & 68.23 & \underline{71.56} & \textbf{82.15} \\ 
 & NMI \textuparrow & 62.63 & 74.51 & 76.89 & 82.87 & \underline{87.72} & \textbf{91.77} \\ 
 & Tokens (M) \textdownarrow & -- & -- & 0.21 & 0.20 & \textbf{0.13} & \underline{0.19} \\
 & Calls \textdownarrow & -- & -- & 649 & 687 & \textbf{393} & \underline{556} \\
 & Cost (\$) \textdownarrow & -- & -- & 0.50 & 0.50 & \textbf{0.32} & \underline{0.50} \\
\hline

\multirow{5}{*}{\shortstack[l]{\textit{Cora} \\ \scriptsize ($B=\$0.5$)}} 
 & FP \textuparrow  & 46.07 & 61.22 & 78.46 & 82.52 & \underline{84.51} & \textbf{86.34} \\ 
 & NMI \textuparrow & 58.85 & 76.91 & 81.12 & 85.28 & \underline{87.16} & \textbf{90.02} \\ 
 & Tokens (M) \textdownarrow & -- & -- & 0.21 & 0.23 & \textbf{0.16} & \underline{0.20} \\
 & Calls \textdownarrow & -- & -- & \underline{685} & 803 & \textbf{427} & 765 \\
 & Cost (\$) \textdownarrow & -- & -- & 0.50 & 0.50 & \textbf{0.31} & \underline{0.50} \\
\hline

\multirow{5}{*}{\shortstack[l]{\textit{AS} \\ \scriptsize ($B=\$1$)}} 
 & FP \textuparrow  & 34.95 & 50.32 & 53.36 & 58.12 & \underline{65.21} & \textbf{71.71} \\ 
 & NMI \textuparrow & 46.28 & 56.63 & 59.05 & 67.43 & \underline{75.73} & \textbf{85.78} \\ 
 & Tokens (M) \textdownarrow & -- & -- & 0.85 & 0.79 & \textbf{0.54} & \underline{0.72} \\
 & Calls \textdownarrow & -- & -- & 2,439 & 2,063 & \textbf{1,162} & \underline{1,978} \\
 & Cost (\$) \textdownarrow & -- & -- & 1.00 & 1.00 & \textbf{0.60} & \underline{1.00} \\
\hline

\multirow{5}{*}{\shortstack[l]{\textit{Amazon-GP} \\ \scriptsize ($B=\$1.5$)}} 
 & FP \textuparrow  & 39.65 & 58.13 & 67.69 & 78.13 & \underline{82.74} & \textbf{87.38} \\ 
 & NMI \textuparrow & 48.18 & 67.70 & 72.36 & 80.17 & \underline{84.79} & \textbf{93.27} \\ 
 & Tokens (M) \textdownarrow & -- & -- & 1.63 & 1.51 & \textbf{0.77} & \underline{1.22} \\
 & Calls \textdownarrow & -- & -- & 4,728 & 4,236 & \textbf{2,153} & \underline{3,997} \\
 & Cost (\$) \textdownarrow & -- & -- & 1.50 & 1.50 & \textbf{0.86} & \underline{1.50} \\
\hline

\multirow{5}{*}{\shortstack[l]{\textit{Song} \\ \scriptsize ($B=\$1$)}} 
 & FP \textuparrow  & 55.67 & 66.35 & 63.41 & 71.65 & \underline{79.74} & \textbf{91.12} \\ 
 & NMI \textuparrow & 64.04 & 73.04 & 69.28 & 74.80 & \underline{82.49} & \textbf{93.73} \\ 
 & Tokens (M) \textdownarrow & -- & -- & 1.73 & 1.51 & \textbf{0.78} & \underline{1.36} \\
 & Calls \textdownarrow & -- & -- & 5,327 & 4,913 & \textbf{2,742} & \underline{4,521} \\
 & Cost (\$) \textdownarrow & -- & -- & 1.00 & 1.00 & \textbf{0.74} & \underline{1.00} \\
\hline

\multirow{5}{*}{\shortstack[l]{\textit{Alaska} \\ \scriptsize ($B=\$3$)}} 
 & FP \textuparrow  & 39.48 & 56.12 & 54.28 & 65.64 & \underline{74.82} & \textbf{79.58} \\ 
 & NMI \textuparrow & 54.75 & 65.91 & 65.77 & 76.36 & \underline{83.44} & \textbf{89.61} \\ 
 & Tokens (M) \textdownarrow & -- & -- & 2.32 & 2.17 & \textbf{1.72} & \underline{2.12} \\
 & Calls \textdownarrow & -- & -- & 6,792 & 8,194 & \textbf{5,917} & \underline{6,574} \\
 & Cost (\$) \textdownarrow & -- & -- & 3.00 & 3.00 & \textbf{2.10} & \underline{3.00} \\
\hline

\multirow{5}{*}{\shortstack[l]{\textit{Music} \\ \scriptsize ($B=\$2$)}} 
 & FP \textuparrow  & 37.14 & 59.39 & 58.38 & 61.21 & \underline{65.91} & \textbf{77.26} \\ 
 & NMI \textuparrow & 55.48 & 70.82 & 68.53 & 72.33 & \underline{76.65} & \textbf{88.05} \\ 
 & Tokens (M) \textdownarrow & -- & -- & 1.59 & 1.43 & \underline{1.37} & \textbf{1.23} \\
 & Calls \textdownarrow & -- & -- & 5,162 & 4,864 & \textbf{3,253} & \underline{3,868} \\
 & Cost (\$) \textdownarrow & -- & -- & 2.00 & 2.00 & 2.00 & 2.00 \\
\hline

\multirow{5}{*}{\shortstack[l]{\textit{Movies} \\ \scriptsize ($B=\$2$)}} 
 & FP \textuparrow  & 31.31 & 45.82 & 43.76 & 47.34 & \underline{58.63} & \textbf{64.68} \\ 
 & NMI \textuparrow & 48.75 & 63.44 & 57.27 & 61.38 & \underline{68.71} & \textbf{77.11} \\ 
 & Tokens (M) \textdownarrow & -- & -- & 1.54 & 1.63 & \underline{1.46} & \textbf{1.42} \\
 & Calls \textdownarrow & -- & -- & 4,632 & 4,872 & \underline{4,352} & \textbf{4,123} \\
 & Cost (\$) \textdownarrow & -- & -- & 2.00 & 2.00 & 2.00 & 2.00 \\
\bottomrule
\end{tabular}
}
\vspace{-2ex}
\end{table}

%% file: figures/query-cost.tex
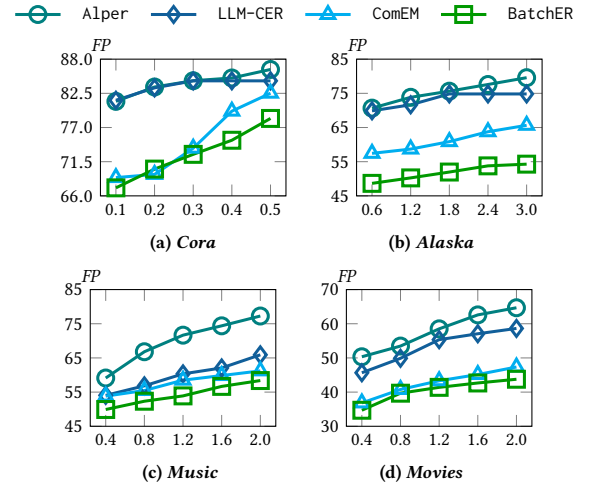
\begin{figure}[!t]
\centering
\begin{small}

\begin{tikzpicture}
    \def\addlegendimage{\csname pgfplots@addlegendimage\endcsname}

    \begin{customlegend}
    [legend columns=4,
        legend entries={\algo, \texttt{LLM-CER}, \texttt{ComEM}, \texttt{BatchER}},
        legend style={at={(0.5,1.35)},anchor=north,draw=none,font=\footnotesize,column sep=0.2cm}]
    \addlegendimage{line width=0.4mm,mark size=3pt,mark=o,color=teal}
    \addlegendimage{line width=0.4mm,mark size=3pt,mark=diamond,color=myblue2}
    \addlegendimage{line width=0.4mm,mark size=3pt,mark=triangle,color=cyan}
    \addlegendimage{line width=0.4mm,mark size=3pt,mark=square,color=green!60!black}
    \end{customlegend}
\end{tikzpicture}
\\[-\lineskip]
\vspace{-3ex}
\subfloat[\em \textit{Cora}]{
\begin{tikzpicture}[scale=1,every mark/.append style={mark size=3pt}]
    \begin{axis}[
        height=\columnwidth/2.5,
        width=\columnwidth/2.1,
        ylabel={\it FP},
        xmin=0.06, xmax=0.54,                %
        xtick={0.1, 0.2, 0.3, 0.4, 0.5},
        xticklabels={0.1, 0.2, 0.3, 0.4, 0.5},
        xticklabel style = {font=\footnotesize},
        ymin=66, ymax=88.0,                    %
        ytick={66,71.5, 77, 82.5, 88.00},      %
        yticklabel style = {font=\footnotesize},
        /pgf/number format/fixed,
        /pgf/number format/precision=1,
        /pgf/number format/fixed zerofill,
        every axis y label/.style={font=\footnotesize,at={(current axis.north west)},right=0mm,above=0mm},
    ]
    \addplot[line width=0.4mm, mark=o,color=teal] 
        plot coordinates { (0.1, 81.21) (0.2, 83.52) (0.3, 84.53) (0.4, 84.97) (0.5, 86.34) };
    \addplot[line width=0.4mm, mark=diamond,color=myblue2] 
        plot coordinates { (0.1, 81.39) (0.2, 83.38) (0.3, 84.51) (0.4, 84.51) (0.5, 84.51) };
    \addplot[line width=0.4mm, mark=triangle,color=cyan] 
        plot coordinates { (0.1, 68.94) (0.2, 69.45) (0.3, 73.67) (0.4, 79.63) (0.5, 82.52) };
    \addplot[line width=0.4mm, mark=square,color=green!60!black] 
        plot coordinates { (0.1, 67.29) (0.2, 70.31) (0.3, 72.68) (0.4, 74.95) (0.5, 78.46) };
    \end{axis}
\end{tikzpicture}\hspace{4mm}\label{fig:res-cora}%
}
\subfloat[\em \textit{Alaska}]{
\begin{tikzpicture}[scale=1,every mark/.append style={mark size=3pt}]
    \begin{axis}[
        height=\columnwidth/2.5,
        width=\columnwidth/2.1,
        ylabel={\it FP},
        xmin=0.36, xmax=3.24,                  %
        xtick={0.6, 1.2, 1.8, 2.4, 3.0},
        xticklabels={0.6, 1.2, 1.8, 2.4, 3.0},
        xticklabel style = {font=\footnotesize},
        ymin=45, ymax=85,                    %
        ytick={45,55, 65, 75, 85},      %
        yticklabel style = {font=\footnotesize},
        every axis y label/.style={font=\footnotesize,at={(current axis.north west)},right=0mm,above=0mm},
    ]
    \addplot[line width=0.4mm, mark=o,color=teal] 
        plot coordinates { (0.6, 70.67) (1.2, 73.84) (1.8, 75.63) (2.4, 77.59) (3.0, 79.58) };
    \addplot[line width=0.4mm, mark=diamond,color=myblue2] 
        plot coordinates { (0.6, 69.88) (1.2, 71.65) (1.8, 74.82) (2.4, 74.82) (3.0, 74.82) };
    \addplot[line width=0.4mm, mark=triangle,color=cyan] 
        plot coordinates { (0.6, 57.44) (1.2, 58.69) (1.8, 60.86) (2.4, 63.76) (3.0, 65.64) };
    \addplot[line width=0.4mm, mark=square,color=green!60!black] 
        plot coordinates { (0.6, 48.66) (1.2, 50.26) (1.8, 51.97) (2.4, 53.78) (3.0, 54.28) };
    \end{axis}
\end{tikzpicture}\hspace{0mm}\label{fig:res-alaska}%
}
\\[-\lineskip]
\vspace{-2ex}
\subfloat[\em \textit{Music}]{
\begin{tikzpicture}[scale=1,every mark/.append style={mark size=3pt}]
    \begin{axis}[
        height=\columnwidth/2.5,
        width=\columnwidth/2.1,
        ylabel={\it FP},
        xmin=0.24, xmax=2.16,                  %
        xtick={0.4, 0.8, 1.2, 1.6, 2.0},
        xticklabels={0.4, 0.8, 1.2, 1.6, 2.0},
        xticklabel style = {font=\footnotesize},
        ymin=45, ymax=85,                    %
        ytick={45,55, 65, 75, 85},      %
        yticklabel style = {font=\footnotesize},
        every axis y label/.style={font=\footnotesize,at={(current axis.north west)},right=0mm,above=0mm},
    ]
    \addplot[line width=0.4mm, mark=o,color=teal] 
        plot coordinates { (0.4, 59.12) (0.8, 66.81) (1.2, 71.66) (1.6, 74.39) (2.0, 77.26) };
    \addplot[line width=0.4mm, mark=diamond,color=myblue2] 
        plot coordinates { (0.4, 54.05) (0.8, 56.84) (1.2, 60.34) (1.6, 62.07) (2.0, 65.91) };
    \addplot[line width=0.4mm, mark=triangle,color=cyan] 
        plot coordinates { (0.4, 53.82) (0.8, 55.47) (1.2, 58.49) (1.6, 59.84) (2.0, 61.21) };
    \addplot[line width=0.4mm, mark=square,color=green!60!black] 
        plot coordinates { (0.4, 49.91) (0.8, 52.33) (1.2, 53.89) (1.6, 56.67) (2.0, 58.38) };
    \end{axis}
\end{tikzpicture}\hspace{4mm}\label{fig:res-music}%
}
\subfloat[\em \textit{Movies}]{
\begin{tikzpicture}[scale=1,every mark/.append style={mark size=3pt}]
    \begin{axis}[
        height=\columnwidth/2.5,
        width=\columnwidth/2.1,
        ylabel={\it FP},
        xmin=0.24, xmax=2.16,                  %
        xtick={0.4, 0.8, 1.2, 1.6, 2.0},
        xticklabels={0.4, 0.8, 1.2, 1.6, 2.0},
        xticklabel style = {font=\footnotesize},
        ymin=30, ymax=70,                    %
        ytick={30, 40, 50, 60,70},      %
        yticklabel style = {font=\footnotesize},
        every axis y label/.style={font=\footnotesize,at={(current axis.north west)},right=0mm,above=0mm},
    ]
    \addplot[line width=0.4mm, mark=o,color=teal] 
        plot coordinates { (0.4, 50.34) (0.8, 53.47) (1.2, 58.51) (1.6, 62.59) (2.0, 64.68) };
    \addplot[line width=0.4mm, mark=diamond,color=myblue2] 
        plot coordinates { (0.4, 45.66) (0.8, 49.89) (1.2, 55.34) (1.6, 57.07) (2.0, 58.63) };
    \addplot[line width=0.4mm, mark=triangle,color=cyan] 
        plot coordinates { (0.4, 36.83) (0.8, 40.77) (1.2, 43.32) (1.6, 45.14) (2.0, 47.34) };
    \addplot[line width=0.4mm, mark=square,color=green!60!black] 
        plot coordinates { (0.4, 34.69) (0.8, 39.68) (1.2, 41.39) (1.6, 42.67) (2.0, 43.76) };
    \end{axis}
\end{tikzpicture}\hspace{0mm}\label{fig:res-movies}%
}
\end{small}
\vspace{-2ex}
\caption{ER performance when varying query budget $B$.} 
\label{fig:query_results}
\vspace{-2ex}
\end{figure}

%% file: figures/ablation.tex
\begin{table}[!t]
\centering
\footnotesize
\renewcommand{\arraystretch}{0.95} 
\setlength{\tabcolsep}{3pt} 
\caption{Ablation study and component analysis of \algo.}
\label{tab:merged-analysis}
\vspace{-2ex}
\begin{tabular}{l cc cc cc cc}
\toprule
\multirow{2}{*}{\textbf{Variant}} & \multicolumn{2}{c}{\textit{Cora}} & \multicolumn{2}{c}{\textit{Alaska}} & \multicolumn{2}{c}{\textit{Music}} & \multicolumn{2}{c}{\textit{Movies}} \\
\cmidrule(lr){2-3} \cmidrule(lr){4-5} \cmidrule(lr){6-7} \cmidrule(lr){8-9}
 & \textbf{FP} & \textbf{NMI} & \textbf{FP} & \textbf{NMI} & \textbf{FP} & \textbf{NMI} & \textbf{FP} & \textbf{NMI} \\
\midrule
\textbf{\algo} & \textbf{86.34} & \textbf{90.02} & \textbf{79.58} & \textbf{89.61} & \textbf{77.26} & \textbf{88.05} & \textbf{64.68} & \textbf{77.17} \\
\midrule
\multicolumn{9}{l}{\textit{Ablation Study}} \\
w/o WLP & 78.69 & 81.67 & 74.67 & 82.35 & \underline{71.49} & \underline{79.68} & \underline{57.72} & \underline{71.39} \\
w/o LLM & \underline{80.23} & \underline{83.18} & \underline{75.36} & \underline{82.93} & 70.61 & 78.84 & 55.43 & 68.47 \\
\midrule
\multicolumn{9}{l}{\textit{Signal Selection}} \\
Greedy & \underline{83.46} & \underline{86.79} & \underline{75.72} & \underline{84.85} & \underline{72.78} & \underline{83.69} & \underline{58.76} & \underline{70.42} \\
Random & 81.81 & 84.33 & 71.49 & 80.17 & 70.34 & 78.63 & 54.34 & 67.37 \\
\midrule
\multicolumn{9}{l}{\textit{Graph Initialization}} \\
\texttt{LSH} & 85.21 & 89.17 & 78.01 & 88.54 & 76.15 & 87.25 & 63.12 & 76.05 \\
\texttt{Canopy} & 84.76 & 88.93 & 78.49 & 88.78 & 75.69 & 86.98 & 63.58 & 76.29 \\
\texttt{Sparkly} & \underline{86.15} & \underline{89.68} & \underline{79.35} & \underline{89.32} & \underline{77.05} & \underline{87.74} & \underline{64.59} & \underline{77.11} \\
\bottomrule
\end{tabular}
\vspace{-1ex}
\end{table}

%% file: figures/llms.tex
\definecolor{mutedblue}{RGB}{114, 147, 184}    %
\definecolor{mutedorange}{RGB}{224, 177, 157}  %
\definecolor{mutedgreen}{RGB}{146, 178, 166}   %
\definecolor{mutedpink}{RGB}{211, 165, 180}    %
\definecolor{mutedcyan}{RGB}{155, 196, 201}    %

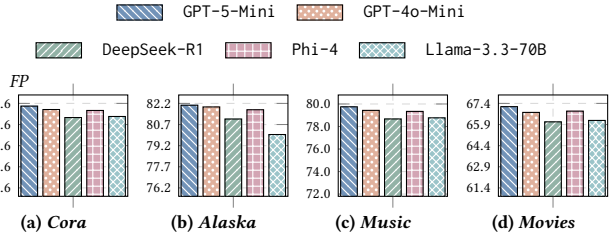
\begin{figure}[!t]
\centering
\begin{small}
\captionsetup[subfloat]{captionskip=-1ex}
\begin{tikzpicture}
    \begin{customlegend}[
        legend entries={{\texttt{GPT-5-Mini}}, {\texttt{GPT-4o-Mini}}},
        legend columns=2,
        area legend,
        legend style={
            at={(0.5, 1.6)},
            anchor=north,
            draw=none,
            font=\footnotesize,
            column sep=0.2cm
        }
    ]
    \addlegendimage{preaction={fill, mutedblue}, pattern=north west lines, pattern color=white}
    \addlegendimage{preaction={fill, mutedorange}, pattern=crosshatch dots, pattern color=white}
    \end{customlegend}

    \begin{customlegend}[
        legend entries={{\texttt{DeepSeek-R1}}, {\texttt{Phi-4}}, {\texttt{Llama-3.3-70B}}},
        legend columns=3,
        area legend,
        legend style={
            at={(0.5, 1.1)},
            anchor=north,
            draw=none,
            font=\footnotesize,
            column sep=0.2cm
        }
    ]
    \addlegendimage{preaction={fill, mutedgreen}, pattern=north east lines, pattern color=white}
    \addlegendimage{preaction={fill, mutedpink}, pattern=grid, pattern color=white}
    \addlegendimage{preaction={fill, mutedcyan}, pattern=crosshatch, pattern color=white}
    \end{customlegend}
\end{tikzpicture}
\\[-\lineskip]
\vspace{-3ex}
\subfloat[\textit{Cora}]{
\begin{tikzpicture}[scale=1]
    \begin{axis}[
        height=\columnwidth/2.9,
        width=\columnwidth/2.8,
        ybar,
        bar width=0.22cm,
        enlarge x limits=0.3,
        symbolic x coords={CORA},
        xtick=data,
        xticklabels={},
        ylabel={\it FP},
        ymin=77.8, ymax=87.4,
        ytick={78.6, 80.6, 82.6, 84.6, 86.6},
        yticklabel style={font=\tiny, /pgf/number format/fixed, /pgf/number format/precision=1, /pgf/number format/fixed zerofill},
        grid=major,
        grid style={dashed, gray!30},
        every axis y label/.style={font=\footnotesize,at={(current axis.north west)},right=20mm,above=0mm},
    ]
    \addplot [preaction={fill, mutedblue}, pattern=north west lines, pattern color=white] coordinates {(CORA, 86.34)};
    \addplot [preaction={fill, mutedorange}, pattern=crosshatch dots, pattern color=white] coordinates {(CORA, 86.02)};
    \addplot [preaction={fill, mutedgreen}, pattern=north east lines, pattern color=white] coordinates {(CORA, 85.26)};
    \addplot [preaction={fill, mutedpink}, pattern=grid, pattern color=white] coordinates {(CORA, 85.93)};
    \addplot [preaction={fill, mutedcyan}, pattern=crosshatch, pattern color=white] coordinates {(CORA, 85.36)};
    \end{axis}
\end{tikzpicture}
}
\hfil
\subfloat[\textit{Alaska}]{
\begin{tikzpicture}[scale=1]
    \begin{axis}[
        height=\columnwidth/2.9,
        width=\columnwidth/2.8,
        ybar,
        bar width=0.22cm,
        enlarge x limits=0.3,
        symbolic x coords={ALASKA},
        xtick=data,
        xticklabels={},
        ylabel={},
        ymin=75.6, ymax=82.8,
        ytick={76.2, 77.7, 79.2, 80.7, 82.2},
        yticklabel style={font=\tiny, /pgf/number format/fixed, /pgf/number format/precision=1, /pgf/number format/fixed zerofill},
        grid=major,
        grid style={dashed, gray!30},
    ]
    \addplot [preaction={fill, mutedblue}, pattern=north west lines, pattern color=white] coordinates {(ALASKA, 82.07)};
    \addplot [preaction={fill, mutedorange}, pattern=crosshatch dots, pattern color=white] coordinates {(ALASKA, 81.95)};
    \addplot [preaction={fill, mutedgreen}, pattern=north east lines, pattern color=white] coordinates {(ALASKA, 81.09)};
    \addplot [preaction={fill, mutedpink}, pattern=grid, pattern color=white] coordinates {(ALASKA, 81.75)};
    \addplot [preaction={fill, mutedcyan}, pattern=crosshatch, pattern color=white] coordinates {(ALASKA, 79.99)};
    \end{axis}
\end{tikzpicture}
}
\subfloat[\textit{Music}]{
\begin{tikzpicture}[scale=1]
    \begin{axis}[
        height=\columnwidth/2.9,
        width=\columnwidth/2.8,
        ybar,
        bar width=0.22cm,
        enlarge x limits=0.3,
        symbolic x coords={MUSIC},
        xtick=data,
        xticklabels={},
        ymin=71.8, ymax=80.8,
        ytick={72.0, 74.0, 76.0, 78.0, 80.0},
        yticklabel style={font=\tiny, /pgf/number format/fixed, /pgf/number format/precision=1, /pgf/number format/fixed zerofill},
        grid=major,
        grid style={dashed, gray!30},
    ]
    \addplot [preaction={fill, mutedblue}, pattern=north west lines, pattern color=white] coordinates {(MUSIC, 79.75)};
    \addplot [preaction={fill, mutedorange}, pattern=crosshatch dots, pattern color=white] coordinates {(MUSIC, 79.43)};
    \addplot [preaction={fill, mutedgreen}, pattern=north east lines, pattern color=white] coordinates {(MUSIC, 78.67)};
    \addplot [preaction={fill, mutedpink}, pattern=grid, pattern color=white] coordinates {(MUSIC, 79.34)};
    \addplot [preaction={fill, mutedcyan}, pattern=crosshatch, pattern color=white] coordinates {(MUSIC, 78.77)};
    \end{axis}
\end{tikzpicture}
}
\hfil
\subfloat[\textit{Movies}]{
\begin{tikzpicture}[scale=1]
    \begin{axis}[
        height=\columnwidth/2.9,
        width=\columnwidth/2.8,
        ybar,
        bar width=0.22cm,
        enlarge x limits=0.3,
        symbolic x coords={MOVIES},
        xtick=data,
        xticklabels={},
        ylabel={},
        ymin=60.8, ymax=68.0,
        ytick={61.4, 62.9, 64.4, 65.9, 67.4},
        yticklabel style={font=\tiny, /pgf/number format/fixed, /pgf/number format/precision=1, /pgf/number format/fixed zerofill},
        grid=major,
        grid style={dashed, gray!30},
    ]
    \addplot [preaction={fill, mutedblue}, pattern=north west lines, pattern color=white] coordinates {(MOVIES, 67.17)};
    \addplot [preaction={fill, mutedorange}, pattern=crosshatch dots, pattern color=white] coordinates {(MOVIES, 66.76)};
    \addplot [preaction={fill, mutedgreen}, pattern=north east lines, pattern color=white] coordinates {(MOVIES, 66.09)};
    \addplot [preaction={fill, mutedpink}, pattern=grid, pattern color=white] coordinates {(MOVIES, 66.85)};
    \addplot [preaction={fill, mutedcyan}, pattern=crosshatch, pattern color=white] coordinates {(MOVIES, 66.19)};
    \end{axis}
\end{tikzpicture}
}

\end{small}
\vspace{-2ex}
\caption{Varying LLMs.}
\label{fig:llm_comparison_new}
\vspace{-3ex}
\end{figure}

%% file: figures/parameters.tex
\begin{figure}[!t]
\centering
\vspace{-2ex}
\begin{small}
\subfloat[\textit{Cora}]{
    \begin{tikzpicture}[scale=1]
    \begin{axis}[
        height=\columnwidth/2.9,
        width=\columnwidth/2.8, %
        ylabel={\it NMI}, %
        xmin=0.12, xmax=1.08, %
        ymin=87.86, ymax=90.74,
        xtick={0.2, 0.4, 0.6, 0.8, 1.0},
        ytick={88.1, 88.7, 89.3, 89.9, 90.5},
        yticklabel style={
            font=\scriptsize,
            /pgf/number format/fixed,
            /pgf/number format/precision=1,
            /pgf/number format/fixed zerofill
        },
        yticklabel style = {font=\tiny},
        every axis label/.style={font=\footnotesize},
        title style={font=\footnotesize},
        grid=major,
        grid style={dashed, gray!30},
        every axis y label/.style={font=\footnotesize,at={(current axis.north west)},right=20mm,above=0mm},
    ]
    \addplot[line width=0.4mm, mark=o, color=teal]
        plot coordinates {
            (0.2, 88.29) (0.4, 88.97) (0.6, 89.64) (0.8, 90.02) (1.0, 90.61)
        };
    \end{axis}
    \end{tikzpicture}
}
\subfloat[\textit{Alaska}]{
    \begin{tikzpicture}[scale=1]
    \begin{axis}[
        height=\columnwidth/2.9,
        width=\columnwidth/2.8,
        xmin=0.12, xmax=1.08,
        ymin=87.8, ymax=90.2,
        xtick={0.2, 0.4, 0.6, 0.8, 1.0},
        ytick={88, 88.5, 89, 89.5, 90.0}, %
        yticklabel style={
            font=\scriptsize,
            /pgf/number format/fixed,
            /pgf/number format/precision=1,
            /pgf/number format/fixed zerofill
        },
        yticklabel style = {font=\tiny},
        every axis label/.style={font=\footnotesize},
        title style={font=\footnotesize},
        grid=major,
        grid style={dashed, gray!30}
    ]
    \addplot[line width=0.4mm, mark=o, color=teal]
        plot coordinates {
            (0.2, 88.24) (0.4, 88.67) (0.6, 89.72) (0.8, 89.61) (1.0, 89.73)
        };
    \end{axis}
    \end{tikzpicture}
}
\subfloat[\textit{Music}]{
    \begin{tikzpicture}[scale=1]
    \begin{axis}[
        height=\columnwidth/2.9,
        width=\columnwidth/2.8,
        xmin=0.12, xmax=1.08,
        ymin=86.3, ymax=88.7,
        xtick={0.2, 0.4, 0.6, 0.8, 1.0},
        ytick={86.5, 87.0, 87.5, 88.0, 88.5}, %
        yticklabel style={
            font=\scriptsize,
            /pgf/number format/fixed,
            /pgf/number format/precision=1,
            /pgf/number format/fixed zerofill
        },
        yticklabel style = {font=\tiny},
        every axis label/.style={font=\footnotesize},
        title style={font=\footnotesize},
        grid=major,
        grid style={dashed, gray!30}
    ]
    \addplot[line width=0.4mm, mark=o, color=teal]
        plot coordinates {
            (0.2, 86.95) (0.4, 87.21) (0.6, 87.69) (0.8, 88.05) (1.0, 88.32)
        };
    \end{axis}
    \end{tikzpicture}
}
\subfloat[\textit{Movies}]{
    \begin{tikzpicture}[scale=1]
    \begin{axis}[
        height=\columnwidth/2.9,
        width=\columnwidth/2.8,
        xmin=0.12, xmax=1.08,
        ymin=75.3, ymax=77.7,
        xtick={0.2, 0.4, 0.6, 0.8, 1.0},
        ytick={75.5, 76, 76.5, 77, 77.5}, %
        yticklabel style={
            font=\scriptsize,
            /pgf/number format/fixed,
            /pgf/number format/precision=1,
            /pgf/number format/fixed zerofill
        },
        yticklabel style = {font=\tiny},
        every axis label/.style={font=\footnotesize},
        title style={font=\footnotesize},
        grid=major,
        grid style={dashed, gray!30}
    ]
    \addplot[line width=0.4mm, mark=o, color=teal]
        plot coordinates {
            (0.2, 75.68) (0.4, 75.98) (0.6, 76.54) (0.8, 77.11) (1.0, 76.89)
        };
    \end{axis}
    \end{tikzpicture}
}
\end{small}
\vspace{-3ex}
\caption{Varying $\alpha$.}
\label{fig:hyperparameter_alpha}
\vspace{-3ex}
\end{figure}
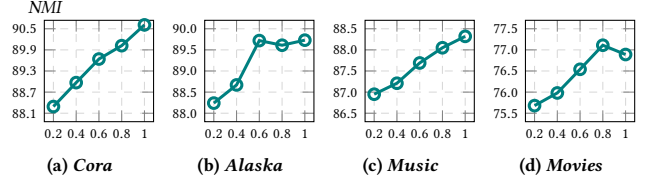

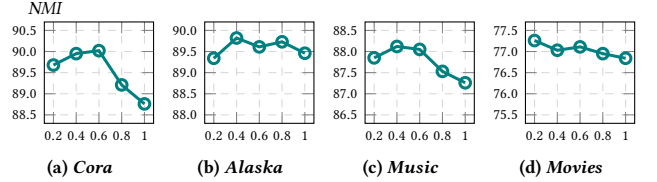
\begin{figure}[!t]
\centering
\vspace{-2ex}
\begin{small}
\subfloat[\textit{Cora}]{
    \begin{tikzpicture}[scale=1]
    \begin{axis}[
        height=\columnwidth/2.9,
        width=\columnwidth/2.8,
        ylabel={\it NMI},
        xmin=0.12, xmax=1.08,
        ymin=88.3, ymax=90.7, %
        xtick={0.2, 0.4, 0.6, 0.8, 1.0},
        ytick={88.5, 89, 89.5, 90, 90.5},
        yticklabel style={
            font=\scriptsize,
            /pgf/number format/fixed,
            /pgf/number format/precision=1,
            /pgf/number format/fixed zerofill
        },
        yticklabel style = {font=\tiny},
        every axis label/.style={font=\footnotesize},
        title style={font=\footnotesize},
        grid=major,
        grid style={dashed, gray!30},
        every axis y label/.style={font=\footnotesize,at={(current axis.north west)},right=20mm,above=0mm},
    ]
    \addplot[line width=0.4mm, mark=o, color=teal]
        plot coordinates {
            (0.2, 89.68) (0.4, 89.95) (0.6, 90.02) (0.8, 89.21) (1.0, 88.76)
        };
    \end{axis}
    \end{tikzpicture}
}
\subfloat[\textit{Alaska}]{
    \begin{tikzpicture}[scale=1]
    \begin{axis}[
        height=\columnwidth/2.9,
        width=\columnwidth/2.8,
        xmin=0.12, xmax=1.08,
        ymin=87.8, ymax=90.2, %
        xtick={0.2, 0.4, 0.6, 0.8, 1.0},
        ytick={88, 88.5, 89, 89.5, 90}, %
        yticklabel style={
            font=\scriptsize,
            /pgf/number format/fixed,
            /pgf/number format/precision=1,
            /pgf/number format/fixed zerofill
        },
        yticklabel style = {font=\tiny},
        every axis label/.style={font=\footnotesize},
        title style={font=\footnotesize},
        grid=major,
        grid style={dashed, gray!30}
    ]
    \addplot[line width=0.4mm, mark=o, color=teal]
        plot coordinates {
            (0.2, 89.34) (0.4, 89.82) (0.6, 89.61) (0.8, 89.73) (1.0, 89.46)
        };
    \end{axis}
    \end{tikzpicture}
}
\subfloat[\textit{Music}]{
    \begin{tikzpicture}[scale=1]
    \begin{axis}[
        height=\columnwidth/2.9,
        width=\columnwidth/2.8,
        xmin=0.12, xmax=1.08,
        ymin=86.3, ymax=88.7,
        xtick={0.2, 0.4, 0.6, 0.8, 1.0},
        ytick={86.5, 87, 87.5, 88, 88.5}, %
        yticklabel style={
            font=\scriptsize,
            /pgf/number format/fixed,
            /pgf/number format/precision=1,
            /pgf/number format/fixed zerofill
        },
        yticklabel style = {font=\tiny},
        every axis label/.style={font=\footnotesize},
        title style={font=\footnotesize},
        grid=major,
        grid style={dashed, gray!30}
    ]
    \addplot[line width=0.4mm, mark=o, color=teal]
        plot coordinates {
            (0.2, 87.85) (0.4, 88.12) (0.6, 88.05) (0.8, 87.53) (1.0, 87.26)
        };
    \end{axis}
    \end{tikzpicture}
}
\subfloat[\textit{Movies}]{
    \begin{tikzpicture}[scale=1]
    \begin{axis}[
        height=\columnwidth/2.9,
        width=\columnwidth/2.8,
        xmin=0.12, xmax=1.08,
        ymin=75.3, ymax=77.7,
        xtick={0.2, 0.4, 0.6, 0.8, 1.0},
        ytick={75.5, 76, 76.5, 77, 77.5}, %
        yticklabel style={
            font=\scriptsize,
            /pgf/number format/fixed,
            /pgf/number format/precision=1,
            /pgf/number format/fixed zerofill
        },
        yticklabel style = {font=\tiny},
        every axis label/.style={font=\footnotesize},
        title style={font=\footnotesize},
        grid=major,
        grid style={dashed, gray!30}
    ]
    \addplot[line width=0.4mm, mark=o, color=teal]
        plot coordinates {
            (0.2, 77.26) (0.4, 77.03) (0.6, 77.11) (0.8, 76.95) (1.0, 76.84)
        };
    \end{axis}
    \end{tikzpicture}
}
\end{small}
\vspace{-3ex}
\caption{Varying $\theta$.}
\label{fig:hyperparameter_theta}
\vspace{-3ex}
\end{figure}

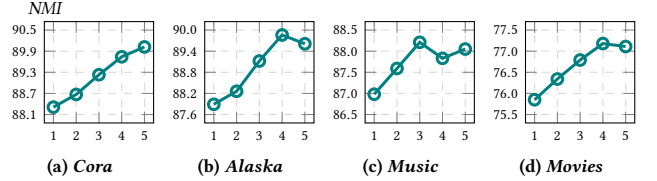
\begin{figure}[!t]
\centering
\vspace{-2ex}
\begin{small}
\subfloat[\textit{Cora}]{
    \begin{tikzpicture}[scale=1]
    \begin{axis}[
        height=\columnwidth/2.9,
        width=\columnwidth/2.8,
        ylabel={\it NMI},
        xmin=0.6, xmax=5.4, %
        ymin=87.86, ymax=90.74,
        xtick={1, 2, 3, 4, 5},
        ytick={88.1, 88.7, 89.3, 89.9, 90.5},
        yticklabel style={
            font=\scriptsize,
            /pgf/number format/fixed,
            /pgf/number format/precision=1,
            /pgf/number format/fixed zerofill
        },
        yticklabel style = {font=\tiny},
        every axis label/.style={font=\footnotesize},
        title style={font=\footnotesize},
        grid=major,
        grid style={dashed, gray!30},
        every axis y label/.style={font=\footnotesize,at={(current axis.north west)},right=20mm,above=0mm},
    ]
    \addplot[line width=0.4mm, mark=o, color=teal]
        plot coordinates {
            (1, 88.31) (2, 88.67) (3, 89.23) (4, 89.74) (5, 90.02)
        };
    \end{axis}
    \end{tikzpicture}
}
\subfloat[\textit{Alaska}]{
    \begin{tikzpicture}[scale=1]
    \begin{axis}[
        height=\columnwidth/2.9,
        width=\columnwidth/2.8,
        xmin=0.6, xmax=5.4,
        ymin=87.36, ymax=90.24, %
        xtick={1, 2, 3, 4, 5},
        ytick={87.60, 88.20, 88.80, 89.40, 90.00}, %
        yticklabel style={
            font=\scriptsize,
            /pgf/number format/fixed,
            /pgf/number format/precision=1,
            /pgf/number format/fixed zerofill
        },
        yticklabel style = {font=\tiny},
        every axis label/.style={font=\footnotesize},
        title style={font=\footnotesize},
        grid=major,
        grid style={dashed, gray!30}
    ]
    \addplot[line width=0.4mm, mark=o, color=teal]
        plot coordinates {
            (1, 87.89) (2, 88.26) (3, 89.12) (4, 89.86) (5, 89.61)
        };
    \end{axis}
    \end{tikzpicture}
}
\subfloat[\textit{Music}]{
    \begin{tikzpicture}[scale=1]
    \begin{axis}[
        height=\columnwidth/2.9,
        width=\columnwidth/2.8,
        xmin=0.6, xmax=5.4,
        ymin=86.3, ymax=88.7,
        xtick={1, 2, 3, 4, 5},
        ytick={86.5, 87, 87.5, 88, 88.5}, %
        yticklabel style={
            font=\scriptsize,
            /pgf/number format/fixed,
            /pgf/number format/precision=1,
            /pgf/number format/fixed zerofill
        },
        yticklabel style = {font=\tiny},
        every axis label/.style={font=\footnotesize},
        title style={font=\footnotesize},
        grid=major,
        grid style={dashed, gray!30}
    ]
    \addplot[line width=0.4mm, mark=o, color=teal]
        plot coordinates {
            (1, 86.98) (2, 87.59) (3, 88.21) (4, 87.83) (5, 88.05)
        };
    \end{axis}
    \end{tikzpicture}
}
\subfloat[\textit{Movies}]{
    \begin{tikzpicture}[scale=1]
    \begin{axis}[
        height=\columnwidth/2.9,
        width=\columnwidth/2.8,
        xmin=0.6, xmax=5.4,
        ymin=75.3, ymax=77.7,
        xtick={1, 2, 3, 4, 5},
        ytick={75.50, 76.00, 76.50, 77.00, 77.50},
        yticklabel style={
            font=\scriptsize,
            /pgf/number format/fixed,
            /pgf/number format/precision=1,
            /pgf/number format/fixed zerofill
        },
        yticklabel style = {font=\tiny},
        every axis label/.style={font=\footnotesize},
        title style={font=\footnotesize},
        grid=major,
        grid style={dashed, gray!30}
    ]
    \addplot[line width=0.4mm, mark=o, color=teal]
        plot coordinates {
            (1, 75.85) (2, 76.34) (3, 76.79) (4, 77.18) (5, 77.11)
        };
    \end{axis}
    \end{tikzpicture}
}
\end{small}
\vspace{-3ex}
\caption{Varying $m$.}
\label{fig:hyperparameter_m}
\vspace{-2ex}
\end{figure}

%% file: tex/add-relatedwork.tex
\section{Additional Related Work}
\label{app:related_work}

In this section, we provide supplementary details on traditional ER methods, the mechanics of crowdsourcing quality control, and widely used blocking techniques.

\subsection{Traditional, Early Learning-based \& Graph-based ER}
Early ER research was grounded in the probabilistic framework introduced by \citet{fellegi1969theory}. Subsequent traditional approaches utilized classical machine learning models (e.g., SVMs, Random Forests)~\cite{konda2016magellan} relying on hand-crafted string similarity features such as Levenshtein~\cite{yujian2007normalized} and Jaro-Winkler distances~\cite{wang2017efficient}. While computationally efficient on small datasets, these methods struggle to capture the semantic ambiguity of heterogeneous entities compared to modern deep learning.
Prior to the widespread adoption of PLMs, GNNs were applied to capture structural dependencies. Models like \texttt{GraphER}~\cite{li2020grapher} and \texttt{HierGAT}~\cite{yao2022entity} learn structural representations to enhance matching decisions. However, largely pre-dating the PLM era, these methods typically lack deep semantic understanding and operate on static graphs, serving primarily as graph-enhanced matchers rather than holistic clusterers. Recent works also attempt to bridge matching and clustering via Expectation-Maximization~\cite{wu2020zeroer}, transitivity closures~\cite{wu2023ground}, or optimizing human oracle allocation~\cite{verroios2017waldo}. However, they fail in modern cost-effective ER because they cannot dynamically repair disconnected topologies or model monetary budgets for paid LLMs.

\subsection{Mechanics of Crowdsourcing-based Approaches}
While Section~\ref{sec:rel_crowd} discusses the optimization strategies of crowdsourcing, this section details the mechanics used to ensure quality. Crowdsourcing approaches leverage human workers as oracles to resolve ambiguities that similarity metrics cannot handle~\cite{howe2006rise}. However, non-expert workers are prone to errors. To address this, robust methods model worker quality using confusion matrices~\cite{verroios2015entity, wang2015crowd} or incorporate ``control queries'' (questions with known answers) to estimate worker reliability dynamically~\cite{galhotra2018robust}. 
Notably, \citet{whang2013question} formulated question selection as an expected accuracy maximization problem, sharing a similar theoretical spirit with the online knapsack formulation used in our work. Despite these robust quality control mechanisms, the reliance on human workers introduces high latency and cost, paving the way for the LLM-driven approaches we advocate for, which offer a faster and more consistent ``pseudo-crowd.''

\subsection{Distinction from Knowledge Base Alignment}
Iterative propagation is also widely studied in Knowledge Base (KB) alignment (e.g., \texttt{PARIS}~\cite{suchanek2011paris}, \texttt{SiGMA}~\cite{lacoste2013sigma}, and \texttt{LINDA}~\cite{bohm2012linda}). However, applying them to flat Dirty ER introduces two critical mismatches. First, they rely on explicit predicate relations in multi-hop graphs to anchor probabilistic updates. On flat tabular schemas lacking this ontological structure, their propagation devolves into a shallow echo chamber of homogeneous similarities, easily amplifying noise. Second, these classical methods are fundamentally budget-agnostic; they assume structural evaluation is computationally free, rendering them unsuited for modern ER workflows governed by strict monetary LLM token constraints.

\subsection{Blocking and Filtering Techniques}
Blocking is an essential preprocessing step to mitigate the quadratic complexity of comparing all record pairs~\cite{papadakis2020blocking}. We categorize common techniques into rule-based, statistical, and learned blocking methods. Early approaches rely on rigid rules. Hash-based methods, such as {\em locality-sensitive hashing} (\texttt{LSH})~\cite{ebraheem2017deeper}, project similar records into the same buckets based on hash collisions. Sort-based methods, like the {\em sorted neighborhood method}~\cite{hernandez1995merge}, cluster records by sorting textual similarity keys and sliding a window over the sorted list. Recent works leverage representation learning to overcome lexical limitations. \texttt{DeepBlocker}~\cite{thirumuruganathan2021deep} utilizes PLMs such as \texttt{SentenceBERT} to generate semantic embeddings, capturing matches missed by syntactic rules. Similarly, \texttt{Sparkly}~\cite{paulsen2023sparkly} employs a distributed top-$k$ blocking framework using TF/IDF. 
It is important to note that all static blocking methods inadvertently sever the connectivity of the entity graph. If a true match is not placed in the same block, it creates disjoint components that downstream stages cannot bridge~\cite{niknam2022role}, a limitation our dynamic graph refinement framework aims to resolve.

%% file: tex/proof.tex
\section{Theoretical Analyses}
\subsection{Proofs}
\begin{proof}[\bf Proof of Lemma~\ref{lem:LP}]
Let $\G$ be the graph of records and $\mathcal{C} = \{\C_1, \C_2, \dots, \C_{|\C|}\}$ be the disjoint clusters of records. The Kronecker function $\delta$ is defined as  
\begin{equation}
\delta_{i,j} = \begin{cases} 1 & \text{if } r_i,r_j \in \C_k, \\
0 & \text{otherwise}, \end{cases}
\end{equation}
which can be derived from the clusters.

According to \cite{raghavan2007near}, each step in label propagation determines the cluster label of the target node $r_i$ in the following way: 
\begin{align*}
k^\ast = \argmax{1\le k\le |\C|}{\sum_{v_j\in \C_k\cap \N(r_i)}{\delta_{i,j}}},
\end{align*}
which is to locally maximize 
$$\sum_{r_j\in \N(r_i)}{w^{+}_{i,j}\cdot \delta_{i,j}}$$ 
for node $r_i$. 
If we consider all nodes in $\mathcal{R}$, it leads to an overall objective of label propagation:
\begin{small}
\begin{align*}
\max_{\C}\sum_{r_i\in \mathcal{R}} \sum_{r_j\in \N(r_i)}{w^{+}_{i,j}\cdot \delta_{i,j}} = \max_{\C}\sum_{r_i,r_j\in \mathcal{R}}{w^{+}_{i,j}\cdot \delta_{i,j}},
\end{align*}
\end{small}
which completes our proof.
\end{proof}

\begin{proof}[\bf Proof of Theorem~\ref{lem:comp-ratio}]
The proof relies on the analysis of the \textit{Online Knapsack Problem} under the assumption of small weights ($\omega_i \ll B$), as established by ~\citet{chakrabarty2008online}.

Let $z \in [0, 1]$ be the fraction of the budget utilized at any given step, $\mathcal{A}$ be the total marginal gain obtained by our algorithm, and $OPT$ be the gain of an optimal offline algorithm. Based on our algorithm's state, we have $z = \beta / B$, where $\beta$ is the currently consumed budget and $B$ is the total budget. We define the value-density threshold function $\Upsilon(z)$ as follows:
\begin{equation}
    \Upsilon(z) = \frac{L}{e} \left( \frac{U \cdot e}{L} \right)^z
\end{equation}
Observing Eq.~\ref{eq:use-LLM} in Section~\ref{sec:signal-select}, our acceptance criterion $\frac{g_i}{\omega_i} \ge \frac{L}{e} (\frac{U \cdot e}{L})^{\frac{\beta}{B}}$ is mathematically equivalent to checking if the marginal value density satisfies $\frac{g_i}{\omega_i} \ge \Upsilon(z)$.

According to Theorem 2.1 in~\cite{chakrabarty2008online}, this specific threshold function $\Upsilon(z)$ is the solution to the differential equation derived from equating the algorithm's performance to the optimal offline allocation in the worst-case scenario. Specifically, for small weights ($\omega_i \ll B$), the cumulative value obtained by the algorithm is approximately $\int_0^{Z} \Upsilon(z) B dz$, where $Z$ is the final budget utilization. The competitive ratio is bounded by:
\begin{equation}
    \frac{OPT}{\mathcal{A}} \le \ln\left(\frac{U}{L}\right) + 1
\end{equation}
Furthermore, Theorem 2.2 in~\cite{chakrabarty2008online} proves that no deterministic or randomized online algorithm can achieve a strictly lower competitive ratio. Thus, our strategy is theoretically optimal under the assumption that $L$ and $U$ correctly bound the real value densities.
\end{proof}

\subsection{Complexity Analysis}
\label{app:complexity}
Let $D$ be the dimension of the embedding vectors and $L_{seq}$ be the max sequence length for PLM encoding.

\stitle{Time Complexity}
The total time consists of the initialization and the iterative refinement.
\begin{itemize}[leftmargin=*]
    \item \textbf{Graph Initialization:} 
    Generating embeddings using a PLM requires $O(n \cdot L_{seq})$. Constructing the initial $K$NN graph using approximate nearest neighbor search (e.g., FAISS) requires $O(n \log n \cdot D)$. Initialization of weights and labels takes $O(n \cdot K)$. The total initialization is dominated by $O(n \log n \cdot D)$.
    
    \item \textbf{Iterative Refinement:} 
    In each iteration $t$:
    (1) \textit{Label Propagation}: For a node $r_i$, LP aggregates signals from its neighbors $\mathcal{N}(r_i)$. 
    A key observation is that $|\mathcal{N}(r_i)|$ is not static. In Algo.~\ref{alg:local-refinement}, identifying a match triggers edge additions to enforce transitive consistency (clique formation).
    Let $\psi_{\text{max}}$ be the maximum size of a cluster maintained by \algo{}.  The cost of one LP pass over all nodes is bounded by $\sum_{i=1}^N |\mathcal{N}(r_i)| \le O(n \cdot \psi_{\text{max}})$.
    (2) \textit{Adaptive Signal Selection:} Computing the Marginal Value Gain (MVG) involves entropy calculations over the label distribution, taking $O(|\mathcal{N}(r_i)|)$ per node. This is similarly bounded by $O(n \cdot \psi_{\text{max}})$.
    (3) \textit{Local Update with LLMs:} Let $\gamma$ be the proportion of nodes selected for LLM querying ($\gamma \ll 1$ due to budget $B$). For a selected node, we retrieve $m$ candidates. The actual LLM API latency is $\tau_\text{LLM}$. The topological update (merging clusters) takes $O(\psi_{\text{max}})$ time.
\end{itemize}

The overall time complexity is:
\begin{equation}
    O\left( n \log n \cdot D + T_\textnormal{max} \cdot n (\psi_{\text{max}} + \gamma \cdot \tau_\text{LLM}) \right).
\end{equation}
Assuming the \textit{sparse truth}~\cite{betancourt2016flexible,draisbach2019transforming} of dirty ER and the high precision of LLMs in preventing excessive false merges, $\psi_{\text{max}}$ can be treated as a small constant relative to $n$ in practice. Thus, \algo{} achieves near-linear scalability $O(n \log n)$ in practice.

\stitle{Space Complexity}
The memory footprint is dominated by the graph structure. As \algo{} performs transitive closure, the space complexity scales with the number of edges. In the worst case, space is $O(n \cdot \psi_{\text{max}} + n \cdot D)$. Given that real-world entity graphs are collections of small, disjoint cliques, memory usage remains linear w.r.t $n$ in standard settings.

\stitle{Monetary Cost Constraint}
Unlike active learning methods that may incur unbounded costs to reach convergence, \algo{} treats monetary cost as a hard constraint. The condition $\beta + \omega \le B$ ensures the cumulative cost never exceeds the pre-defined budget $B$. Furthermore, Algo.~\ref{alg:llm-select} ensures the budget is spent efficiently, preventing premature depletion on low-value queries. Consequently, the monetary cost is strictly bounded by $O(B)$.

%% file: tex/add-exp.tex
\section{Additional Experimental Details}
\label{sec:appendix_exp_setup}

\subsection{Datasets}
\label{app:datasets}

We utilize eight real-world dirty entity resolution benchmarks spanning diverse domains, including demographics, bibliographic citations, geography, software, music, and e-commerce. These datasets exhibit significant variations in scale, entity dispersion, and data quality. Detailed statistics are provided in Table~\ref{tab:datasets}. Below, we describe the specific characteristics of each dataset:
\begin{enumerate}

    \item \textit{Cora}~\cite{nikoletos2022pyjedai}: A widely used homogeneous dataset containing machine learning papers. It is characterized by high entity dispersion ($E_d \approx 11.56$), meaning it contains small clusters of duplicates with significant textual noise in citation fields (e.g., authors, titles, venues).
    \item \textit{Census}~\cite{nikoletos2022pyjedai}: A demographic dataset sourced from the JedAI toolkit~\cite{papadakis2018return}. It comprises records of individuals with attributes such as first name, last name, and address. Despite its small size ($N=841$), it presents a high density of duplicates relative to the number of unique entities.

    \item \textit{Alaska}~\cite{crescenzi2021alaska}: A challenging, large-scale dataset ($N \approx 12k$) containing camera product specifications collected from multiple e-commerce platforms. Originally introduced in the SIGMOD 2020 programming contest, it requires resolving heterogeneous attributes and handling high variability in product descriptions.
    \item \textit{Amazon-GP}~\cite{papadakis2018return}: A software product dataset matching records between Amazon and Google Play. This dataset is domain-specific, requiring the resolution of software applications where slight variations in version numbers or developer names can distinguish unique entities.

    \item \textit{Song}~\cite{saeedi2017comparative}: A music dataset aggregating track information (e.g., artist, album, release year) from various services. It features moderate dispersion ($E_d \approx 4.06$) and typical data quality issues such as inconsistent abbreviations and missing values.
    \item \textit{Music}~\cite{saeedi2017comparative}: Also referred to as \textit{Music20K} in related literature~\cite{fu2025context}, this is a larger-scale music dataset ($N \approx 19k$). It involves resolving musical entities with a mix of textual and categorical attributes, exhibiting a lower dispersion ($E_d \approx 1.94$) compared to \textit{Song}.
    \item \textit{Movies}~\cite{crescenzi2021alaska}: The largest dataset in our benchmark ($N > 50k$), sourced from the JedAI repository~\cite{papadakis2018return}. It contains movie metadata (e.g., title, director, year, cast) and represents a scalability test for ER approaches due to its size and the presence of many non-duplicate entities (singletons).

    \item \textit{AS} (Autonomous Systems)~\cite{papadakis2018return}: A geographical dataset containing information about internet autonomous systems. It involves resolving organization names and locations, often requiring the identification of varying representations of the same geographic or corporate entity.
\end{enumerate}
\noindent \textbf{Data Availability:} The \textit{Cora}, \textit{Alaska}, \textit{Song}, and \textit{Music} datasets are aligned with standard benchmarks used in recent LLM-based ER studies~\cite{fu2025context}. The \textit{Census}, \textit{Amazon-GP}, and \textit{Movies} datasets are obtained from the JedAI Toolkit repository\footnote{\url{https://github.com/scify/JedAIToolkit}}.

\subsection{Baselines}
\label{app:baselines}
We compare \algo against the following state-of-the-art approaches:
\textit{Unsupervised and Self-supervised Approaches:}
\begin{enumerate}
    \item \texttt{ZeroER}~\cite{wu2020zeroer}: An unsupervised, generative approach based on Gaussian Mixture Models (GMMs). It identifies duplicate entities by distinguishing the specific feature distributions of matching and non-matching pairs.
    \item \texttt{CollaborEM}~\cite{ge2021collaborem}: A self-supervised framework that utilizes an automatic label generation strategy to create pseudo-labels. It employs a collaborative ER training phase to integrate graph and sentence features, achieving performance comparable to supervised methods without human annotation.
\end{enumerate}
\textit{LLM-based Approaches:}
\begin{enumerate}
    \item \texttt{BatchER}~\cite{fan2024cost}: A cost-effective method that groups multiple candidate pairs into a single prompt for batch verification, significantly reducing API calls compared to single-pair queries.
    \item \texttt{ComEM}~\cite{wang2025match}: Moves beyond simple pairwise classification by employing a ``match-compare-select'' strategy. It utilizes the reasoning capabilities of LLMs to select the best match from a set of candidates, aiming to minimize transitive errors.
    \item \texttt{LLM-CER}~\cite{fu2025context}: The current state-of-the-art method that performs in-context clustering. It directly groups records into clusters within the prompt to maintain local consistency, rather than relying solely on pairwise decisions.
\end{enumerate} 

All experiments are conducted on a Linux machine with an NVIDIA A100 GPU (80GB RAM), AMD EPYC 7513 CPU (2.6 GHz), and 1TB RAM.

\subsection{Evaluation Protocol}
\label{app:protocol}
\textbf{\textit{FP-measure} (FP).}
This metric evaluates clustering results based on homogeneity and stability. It is defined as the harmonic mean of \textit{purity} and \textit{inverse-purity}. Let $\mathcal{X} = \{X_1, \dots, X_n\}$ denote the predicted clusters and $\mathcal{Y} = \{Y_1, \dots, Y_m\}$ represent the ground truth clusters. The total number of records is $|\mathcal{R}| = \sum_i |X_i| = \sum_j |Y_j|$.
The \textit{purity} and \textit{inverse-purity} are calculated as follows:
\begin{equation}
    \textsf{purity}(\mathcal{X}, \mathcal{Y}) = \sum_{X_i \in \mathcal{X}} \frac{|X_i|}{|\mathcal{R}|} \max_{Y_j \in \mathcal{Y}} \textsf{Overlap}(X_i, Y_j)
\end{equation}
\begin{equation}
    \textsf{inverse-purity}(\mathcal{X}, \mathcal{Y}) = \sum_{Y_j \in \mathcal{Y}} \frac{|Y_j|}{|\mathcal{R}|} \max_{X_i \in \mathcal{X}} \textsf{Overlap}(Y_j, X_i)
\end{equation}
where $\textsf{Overlap}(X_i, Y_j) = |X_i \cap Y_j| / |X_i|$. The FP-measure is then derived as:
\begin{equation}
    \text{FP}(\mathcal{X}, \mathcal{Y}) = \frac{2}{1/\textsf{purity}(\mathcal{X}, \mathcal{Y}) + 1/\textsf{inverse-purity}(\mathcal{X}, \mathcal{Y})}
\end{equation}

\textbf{\textit{Normalized Mutual Information} (NMI).}
NMI measures the similarity between clustering results and ground truth from an entropy perspective:
\begin{equation}
    \text{NMI}(\mathcal{X}, \mathcal{Y}) = \frac{2I(\mathcal{X}, \mathcal{Y})}{H(\mathcal{X}) + H(\mathcal{Y})}
\end{equation}
where $I(\cdot)$ denotes the mutual information shared between clusters and $H(\cdot)$ represents the entropy, quantifying the uncertainty of elements within clusters.

\textbf{\textit{Monetary Cost}.}
To evaluate the cost-effectiveness of \algo against competitors, we estimate the monetary cost incurred by each method. 
For LLM-based approaches (\algo, \texttt{BatchER}, \texttt{ComEM}, and \texttt{LLM-CER}), the cost is determined by the commercial pricing of the LLM APIs. We utilize \texttt{GPT-5-Mini} as the unified backbone. The cost is calculated based on the number of processed tokens:
\begin{equation}\label{eq:API-cost}
    \text{Cost} = \frac{N_{in}}{1\text{M}} \times P_{in} + \frac{N_{out}}{1\text{M}} \times P_{out}
\end{equation}
where $N_{in}$ and $N_{out}$ correspond to the number of input and output tokens, respectively. Based on the official pricing\footnote{\url{https://azure.microsoft.com/en-us/pricing/details/azure-openai}}, we set $P_{in} = \$0.25$ and $P_{out} = \$2.00$ per 1 million tokens.

\subsection{Implementation Details}
\label{app:implementation}

To ensure a fair and reproducible comparison, we use the original implementations and default hyperparameters from the public code repositories listed in Table~\ref{tab:baselines}, with minor adaptations to align with the experimental setting of \algo. For \texttt{BatchER}, the original implementation~\cite{fan2024cost} relies on in-context learning with 8 labeled demonstrations per prompt. To align with the unsupervised zero-shot setting of our approach, we removed these demonstrations. However, we strictly maintained the original batch structure, packing 5 pairwise verification questions (i.e., 10 records) into a single prompt to preserve its design logic. For \texttt{ComEM}, we employed the ``select'' strategy (match-compare-select) as detailed in~\cite{wang2025match}, which leverages the reasoning capabilities of LLMs to identify the most appropriate match from a set of candidates. For \texttt{LLM-CER}, we adhered to the strategy and parameter configurations specified in the original paper~\cite{fu2025context}, specifically utilizing the 9-record clustering prompt design. Furthermore, to strictly evaluate cost-effectiveness under resource constraints, we applied the identical budget-aware control mechanism across all LLM-based methods (\algo, \texttt{BatchER}, \texttt{ComEM}, and \texttt{LLM-CER}), ensuring that all methods operate within the same monetary limits.

\begin{table}[H]
\centering
\caption{Code repositories for baselines.}
\label{tab:baselines}
\vspace{-2ex}
\small
\resizebox{\columnwidth}{!}{%
\begin{tabular}{ll}
\toprule
\textbf{Baseline} & \textbf{Code Repository} \\
\midrule
\texttt{ZeroER} & \url{https://github.com/chu-data-lab/zeroer} \\
\texttt{CollaborEM} & \url{https://github.com/ZJU-DAILY/CollaborEM} \\
\texttt{BatchER} & \url{https://github.com/fmh1art/BatchER} \\
\texttt{ComEM} & \url{https://github.com/tshu-w/ComEM} \\
\texttt{LLM-CER} & \url{https://github.com/ZJU-DAILY/LLMCER} \\
\bottomrule
\end{tabular}%
}
\end{table}

\subsection{Hyperparameters}

Across all experimental settings, we employ \texttt{GPT-5-Mini} as the backbone for $\textsf{LLM}(\cdot)$ and use \texttt{SBERT} to generate initial record embeddings. To ensure a unified evaluation, we fix the structural and algorithmic hyperparameters for \algo{} across all datasets. Specifically, we set the blocking size for the initial $K$NN graph construction to $K=15$ and the number of candidate records for local LLM updates to $m=5$. Regarding the adaptive signal selection and propagation logic, the edge weight scaling factor is set to $\alpha=0.8$, the confidence threshold for weak label propagation is $\theta=0.6$, the expected LLM confidence is $\Delta^\text{\scriptsize{LLM}}_i=0.95$, and the LLM verified weight is $\sigma_\text{LLM}=1.0$. For the online knapsack formulation, the value density bounds are set to $[L, U] = [20, 1000]$. The only dataset-specific hyperparameter is the total monetary query budget $B$, which is allocated based on the scale and complexity of the dataset: \textit{Census} (\$0.5), \textit{Cora} (\$0.5), \textit{AS} (\$1.0), \textit{Song} (\$1.0), \textit{Amazon-GP} (\$1.5), \textit{Music} (\$2.0), \textit{Movies} (\$2.0), and \textit{Alaska} (\$3.0).

\subsection{Parameter Analysis}

\stitle{Impact of Blocking Parameter $K$}
The parameter $K$ determines the initial sparsity of the graph, representing the trade-off between recall potential and computational complexity. Table~\ref{tab:k_parameter_sensitivity} reports the performance of \algo and LLM-based baselines as $K$ varies from 5 to 25.
As expected, increasing $K$ generally improves performance across all methods by introducing more potential true matches into the candidate pool. However, \algo consistently outperforms competitors across all sparsity levels. Notably, even at a low $K=5$, where the graph is sparse and disconnected, \algo achieves an FP score of 81.23\% on \textit{Cora}, surpassing \texttt{LLM-CER} (80.10\%) and \texttt{BatchER} (73.48\%). This demonstrates \algo's ability to recover missing links via transitive propagation, making it less dependent on dense initial blocking than BMC pipelines.

\stitle{Robustness of LLM Confidence Estimation}
In \algo, $\Delta^\text{\scriptsize{LLM}}_i$ represents the confidence of the LLM oracle (set to 0.95 by default). Prior works adopt costly estimation strategies: \cite{huang2025thriftllm} estimates oracle accuracy using an 80\% data split, while \cite{li2024leveraging} samples a subset (e.g., 100 pairs) for calibration.
To validate the fixed parameter, we follow the sampling strategy to calculate an empirical $\Delta^\text{\scriptsize{LLM}}_i$ using 100 sampled questions per dataset.
Table~\ref{tab:xi-ablation-wide} compares the results using the empirical estimate versus the fixed default. The performance gap is negligible (e.g., a mere 0.13\% FP difference on \textit{Cora}).
Crucially, previous studies~\cite{huang2025thriftllm,li2024leveraging} often treat the estimation cost as external to the query budget. In contrast, \algo treats the monetary budget as a hard constraint. By using a robust default value, \algo avoids the sunk cost of estimation, allowing the entire budget to be utilized for actual entity resolution, thereby maximizing cost-effectiveness.

\input{figures/blocking-k}

\input{figures/llm-utility}

\stitle{Sensitivity of LLM Verification Weight}
The parameter $\sigma_{\text{LLM}}$ controls the influence of the LLM oracle during the label propagation process. Setting it too low diminishes the active oracle's ability to correct propagation paths, whereas setting it excessively high makes the graph brittle to potential LLM hallucinations. Table~\ref{tab:beta_sensitivity} demonstrates that a moderate-to-high weight ($\sigma_{\text{LLM}} = 1.0$) achieves the optimal balance across datasets, validating our robust design choices.

\begin{table}[H]
\centering
\caption{Impact of $\sigma_{\text{LLM}}$ on different datasets.}
\label{tab:beta_sensitivity}
\vspace{-2ex}
\small
\begin{tabular}{lcccc}
\toprule
\multirow{2}{*}{\textbf{ $\sigma_{\text{LLM}}$}} & \multicolumn{2}{c}{\em \textit{Cora}} & \multicolumn{2}{c}{\em \textit{Alaska}} \\
\cmidrule(lr){2-3} \cmidrule(lr){4-5}
& \textbf{FP} & \textbf{NMI} & \textbf{FP} & \textbf{NMI} \\
\midrule
0.5 & 84.81 & 88.63 & 78.74 & 89.23 \\
1.0 & \textbf{86.34} & \textbf{90.02} & \textbf{79.58} & \textbf{89.61} \\
1.5 & 86.29 & 89.78 & 79.62 & 89.37 \\
2.0 & 86.05 & 89.69 & 79.16 & 89.29 \\
\bottomrule
\end{tabular}
\end{table}

\subsection{Comparisons with Supervised Baselines}
To broaden our evaluation, we compare \algo against two widely used supervised PLM-based methods, \texttt{Ditto} and \texttt{DeepMatcher}, under a low-resource setting with a 1:9 data split (10\% training, 90\% testing).
As shown in Table~\ref{tab:supervised_comparison_wide}, \algo demonstrates superior performance (e.g., +20.47\% FP on \textit{Movies}).
While supervised methods often struggle to generalize with limited training data, \algo achieves state-of-the-art results without requiring extensive annotated examples, highlighting its practicality for cold-start scenarios.

\input{figures/supervised}

\subsection{Comparison with Legacy KB-Alignment Baselines}
To investigate the performance of iterative similarity propagation methods originally designed for KB alignment, we evaluate \texttt{PARIS}~\cite{suchanek2011paris} and an adapted version of \texttt{LINDA}~\cite{bohm2012linda} on flat Dirty ER datasets. To ensure a fair structural comparison, we test them under two settings: ``Unconstrained'' (their native global scope) and ``Constrained'' (restricted to \algo{}'s exact \texttt{SBERT} Top-15 blocking graph). 

As shown in Table~\ref{tab:kb_baselines}, these methods suffer catastrophic performance degradation. Because flat tabular datasets lack explicit, multi-hop predicate relations (e.g., \textit{is-advisor-of}) required to anchor probabilistic updates, their propagation devolves into a shallow literal comparator. This structural mismatch fragments the datasets into singletons, which is directly reflected in the exceptionally low FP-measure. This confirms that legacy KB-alignment methods are structurally unsuited for flat ER tasks.

\begin{table}[H]
\centering
\caption{Comparison with KB-Alignment methods.}
\label{tab:kb_baselines}
\vspace{-2ex}
\small
\begin{tabular}{llcc}
\toprule
\textbf{Dataset} & \textbf{Method} & \textbf{Mode} & \textbf{FP} \\
\midrule
\multirow{5}{*}{\em \textit{Cora}} & \texttt{PARIS} & Unconstrained & 8.65 \\
& \texttt{PARIS} & Constrained & 8.65 \\
& \texttt{LINDA} & Unconstrained & 35.44 \\
& \texttt{LINDA} & Constrained & 30.93 \\
& \algo{} & Constrained & \textbf{86.34} \\
\midrule
\multirow{5}{*}{\em \textit{Alaska}} & \texttt{PARIS} & Unconstrained & 12.36 \\
& \texttt{PARIS} & Constrained & 12.36 \\
& \texttt{LINDA} & Unconstrained & 43.98 \\
& \texttt{LINDA} & Constrained & 43.98 \\
& \algo{} & Constrained & \textbf{79.58} \\
\bottomrule
\end{tabular}
\end{table}

\subsection{Robustness to Initial PLM Encoders}
While the quality of the initial $K$NN graph naturally impacts the starting point of the label propagation, \algo{} is designed to be highly robust to initial graph sparsity. We evaluate different PLM encoders, as shown in Table~\ref{tab:plm_encoders}. Regardless of the base encoder's initial representation quality, \algo{} successfully recovers missing topological links through LLM-informed propagation, consistently achieving high downstream performance.

\begin{table}[H]
\centering
\caption{Performance across different PLM encoders.}
\label{tab:plm_encoders}
\vspace{-2ex}
\small
\begin{tabular}{llcc}
\toprule
\textbf{Dataset} & \textbf{Encoder} & \textbf{FP} & \textbf{NMI} \\
\midrule
\multirow{3}{*}{\em \textit{Cora}} & \texttt{SBERT} & \textbf{86.34} & \textbf{90.02} \\
& \texttt{RoBERTa-base} & 85.29 & 88.96 \\
& \texttt{BERT-base-uncased} & 84.67 & 88.12 \\
\midrule
\multirow{3}{*}{\em \textit{Alaska}} & \texttt{SBERT} & \textbf{79.58} & \textbf{89.61} \\
& \texttt{RoBERTa-base} & 76.49 & 85.53 \\
& \texttt{BERT-base-uncased} & 74.14 & 82.66 \\
\bottomrule
\end{tabular}
\end{table}

\subsection{Token-Cost Parity Analysis}
To ensure a strictly fair evaluation of LLM overhead, we investigate whether competitors can bridge the performance gap if granted an equivalent or greater token budget. As shown in Table~\ref{tab:token_parity}, we significantly increase the candidate pool size for \texttt{LLM-CER} to match or exceed \algo{}'s token consumption. Despite utilizing massive token budgets, \texttt{LLM-CER} continues to fall short because it remains constrained by static, disjoint blocks. In contrast, \algo{} leverages its global graph to dynamically hunt for high-value boundary cases.

\begin{table}[H]
\centering
\caption{Comparison under matched Token-Cost parity.}
\label{tab:token_parity}
\vspace{-2ex}
\small
\begin{tabular}{llccc}
\toprule
\textbf{Dataset} & \textbf{Method} & \textbf{Token (M)} & \textbf{FP} & \textbf{NMI} \\
\midrule
\multirow{2}{*}{\em \textit{Cora}} 
& \texttt{LLM-CER} & 180.23 & 85.03 & 87.85 \\
& \algo{} & 150.20 & \textbf{86.34} & \textbf{90.02} \\
\midrule
\multirow{2}{*}{\em \textit{Alaska}} 
& \texttt{LLM-CER} & 192.09 & 75.39 & 83.96 \\
& \algo{} & 152.12 & \textbf{79.58} & \textbf{89.61} \\
\midrule
\multirow{2}{*}{\em \textit{Music}} 
& \texttt{LLM-CER} & 141.25 & 65.67 & 76.29 \\
& \algo{} & 151.23 & \textbf{77.26} & \textbf{88.05} \\
\midrule
\multirow{2}{*}{\em \textit{Movies}} 
& \texttt{LLM-CER} & 151.46 & 58.63 & 68.71 \\
& \algo{} & 151.42 & \textbf{64.68} & \textbf{77.11} \\
\bottomrule
\end{tabular}
\end{table}

\subsection{Theoretical Upper Bound with Ground-Truth Oracle}
To assess our framework's theoretical upper bound and isolate the specific performance degradation caused by LLM hallucinations, we replace the LLM backend with a Ground-Truth (True) Oracle. As detailed in Table~\ref{tab:oracle}, the True Oracle reveals the absolute limit of our graph propagation mechanism. The gap between the LLM Oracle and the True Oracle indicates that \algo{} successfully extracts maximum value from imperfect LLMs, while leaving ample room for future backbone upgrades to further close this gap.

\begin{table}[H]
\centering
\caption{Performance limits with a Ground-Truth Oracle.}
\label{tab:oracle}
\vspace{-2ex}
\small
\begin{tabular}{llcc}
\toprule
\textbf{Dataset} & \textbf{Variant} & \textbf{FP} & \textbf{NMI} \\
\midrule
\multirow{3}{*}{\em \textit{Cora}} & WLP Only (No Oracle) & 80.23 & 83.18 \\
& LLM Oracle (\algo{}) & 86.34 & 90.02 \\
& True Oracle & \textbf{93.11} & \textbf{96.13} \\
\midrule
\multirow{3}{*}{\em \textit{Alaska}} & WLP Only (No Oracle) & 75.36 & 82.93 \\
& LLM Oracle (\algo{}) & 79.58 & 89.61 \\
& True Oracle & \textbf{83.24} & \textbf{91.04} \\
\bottomrule
\end{tabular}
\end{table}

\subsection{Empirical Efficiency and Scalability}
To provide a clear picture of \algo{}'s empirical efficiency, we track the end-to-end running times across various datasets. Despite the iterative nature of the label propagation and the sequential API calls, Table~\ref{tab:efficiency} demonstrates that the selective logic of our Online Knapsack formulation keeps the computational overhead strictly bounded, maintaining runtime parity with standard cascaded pipelines like \texttt{LLM-CER}.

\begin{table}[H]
\centering
\caption{End-to-end running times (in minutes).}
\label{tab:efficiency}
\vspace{-2ex}
\small
\begin{tabular}{lrr}
\toprule
\textbf{Dataset} & \textbf{\texttt{LLM-CER}} & \textbf{\algo{}} \\
\midrule
\em \textit{Cora} & 9.02 & 11.75 \\
\em \textit{Alaska} & 91.61 & 105.96 \\
\em \textit{Music} & 59.47 & 56.42 \\
\em \textit{Movies} & 75.32 & 67.68 \\
\bottomrule
\end{tabular}
\end{table}

\subsection{Analysis of Error Correction}
A key advantage of \algo{} is its ability to mitigate early-stage LLM errors. ``Cheap'' LLMs risk generating false positives, but our framework absorbs these through dense structural neighborhoods. Table~\ref{tab:qualitative} tracks a specific boundary case in the \textit{Movies} dataset where an LLM falsely matched two distinct entities in Iteration 1. Because the propagation relies on the weighted consensus of a node's broader neighborhood, the correct structural signals quickly overrode the single erroneous LLM edge in subsequent iterations, preventing a catastrophic cluster merge. Our broader tracking confirms that approximately 12.26\% of such early-stage LLM false positives across the \textit{Movies} dataset were successfully neutralized by this mechanism.

\begin{table}[H]
\centering
\caption{Target label probabilities.}
\label{tab:qualitative}
\vspace{-2ex}
\small
\begin{tabular}{cccl}
\toprule
\textbf{Iteration} & \textbf{Label 1} & \textbf{Label 2} & \textbf{Status} \\
\midrule
1 & 0.45 & 0.26 & Error \\
2 & 0.32 & 0.41 & Recovered \\
3 & 0.13 & 0.68 & Corrected \\
\bottomrule
\end{tabular}
\end{table}

\section{Prompt Template}
\label{app:prompt}

We provide the detailed prompt template utilized by \algo.
As shown in Figure~\ref{fig:entity_resolution_prompt}, this prompt compels the LLM to act as a discriminative expert, comparing a target record against a set of candidates. We enforce a strict output format (integer index or ``NONE'') to ensure automated parsability. This one-to-many verification facilitates comparative reasoning to identify the best match or reject all candidates, thereby refining graph topology and mitigating error propagation.

\input{tex/prompt}

%% file: figures/blocking-k.tex
\begin{table}[!t]
\centering
\begin{small}
\caption{Performance when varying blocking parameter $K$.}
\label{tab:k_parameter_sensitivity}
\vspace{-3ex}
\renewcommand{\arraystretch}{0.95}
\setlength{\tabcolsep}{2pt}
\resizebox{\columnwidth}{!}{%
\begin{tabular}{c|c|cc|cc|cc|cc}
\toprule
\multirow{2}{*}{\textbf{Setting}} & \multirow{2}{*}{\textbf{Method}} 
& \multicolumn{2}{c|}{\em \textit{Cora}} 
& \multicolumn{2}{c|}{\em \textit{Alaska}} 
& \multicolumn{2}{c|}{\em \textit{Music}} 
& \multicolumn{2}{c}{\em \textit{Movies}} \\
\cline{3-10}
& & \textbf{FP} & \textbf{NMI} 
& \textbf{FP} & \textbf{NMI} 
& \textbf{FP} & \textbf{NMI} 
& \textbf{FP} & \textbf{NMI} \\
\hline
\multirow{4}{*}{\rotatebox[origin=c]{90}{$K=5$}} 
& \texttt{BatchER} & 73.48 & 76.96 & 49.30 & 61.61 & 53.40 & 64.37 & 41.78 & 53.11 \\
& \texttt{ComEM}   & 78.51 & 82.08 & 61.63 & 73.16 & 57.20 & 69.13 & 43.33 & 58.18 \\
& \texttt{LLM-CER}  & \underline{80.10} & \underline{84.87} & \underline{70.41} & \underline{81.15} & \underline{61.50} & \underline{74.36} & \underline{54.22} & \underline{66.42} \\
& \algo   & \textbf{81.23} & \textbf{86.49} & \textbf{74.47} & \textbf{86.08} & \textbf{72.15} & \textbf{84.52} & \textbf{59.57} & \textbf{73.58} \\
\cmidrule(lr){1-10}

\multirow{4}{*}{\rotatebox[origin=c]{90}{$K=10$}} 
& \texttt{BatchER} & 77.56 & 79.72 & 53.66 & 64.63 & 57.71 & 67.35 & 46.22 & 56.28 \\
& \texttt{ComEM}   & 81.37 & 84.26 & 64.73 & 75.45 & 60.36 & 71.46 & 46.68 & 60.65 \\
& \texttt{LLM-CER}  & \underline{82.64} & \underline{86.53} & \underline{73.16} & \underline{82.84} & \underline{64.45} & \underline{76.10} & \underline{57.33} & \underline{68.21} \\
& \algo   & \textbf{84.95} & \textbf{88.41} & \textbf{78.30} & \textbf{88.01} & \textbf{76.02} & \textbf{86.48} & \textbf{63.64} & \textbf{75.73} \\
\cmidrule(lr){1-10}

\multirow{4}{*}{\rotatebox[origin=c]{90}{$K=15$}} 
& \texttt{BatchER} & 78.46 & 81.12 & 54.28 & 65.77 & 58.38 & 68.53 & 46.76 & 57.27 \\
& \texttt{ComEM}   & 82.52 & 85.28 & 65.64 & 76.36 & 61.21 & 72.33 & 47.34 & 61.38 \\
& \texttt{LLM-CER}  & \underline{84.51} & \underline{87.16} & \underline{74.82} & \underline{83.44} & \underline{65.91} & \underline{76.65} & \underline{58.63} & \underline{68.71} \\
& \algo   & \textbf{86.34} & \textbf{90.02} & \textbf{79.58} & \textbf{89.61} & \textbf{77.26} & \textbf{88.05} & \textbf{64.68} & \textbf{77.11} \\
\cmidrule(lr){1-10}

\multirow{4}{*}{\rotatebox[origin=c]{90}{$K=20$}} 
& \texttt{BatchER} & 78.86 & 81.95 & 54.56 & 66.44 & 58.68 & 69.23 & 47.00 & 57.86 \\
& \texttt{ComEM}   & 83.02 & 85.91 & 66.04 & 76.92 & 61.58 & 72.86 & 47.63 & 61.83 \\
& \texttt{LLM-CER}  & \underline{85.28} & \underline{88.03} & \underline{75.50} & \underline{84.27} & \underline{66.51} & \underline{77.42} & \underline{59.16} & \underline{69.40} \\
& \algo   & \textbf{86.76} & \textbf{90.57} & \textbf{79.97} & \textbf{90.16} & \textbf{77.64} & \textbf{88.59} & \textbf{64.99} & \textbf{77.58} \\
\cmidrule(lr){1-10}

\multirow{4}{*}{\rotatebox[origin=c]{90}{$K=25$}} 
& \texttt{BatchER} & 80.34 & 85.25 & 56.16 & 69.90 & 60.26 & 72.66 & 48.64 & 61.40 \\
& \texttt{ComEM}   & 84.16 & 86.34 & 67.28 & 77.42 & 62.85 & 73.39 & 48.98 & 62.44 \\
& \texttt{LLM-CER}  & \underline{86.31} & \underline{89.65} & \underline{76.62} & \underline{85.93} & \underline{67.71} & \underline{79.14} & \underline{60.43} & \underline{71.20} \\
& \algo   & \textbf{88.83} & \textbf{92.42} & \textbf{82.07} & \textbf{92.01} & \textbf{79.75} & \textbf{90.45} & \textbf{67.17} & \textbf{79.51} \\
\bottomrule
\end{tabular}
}
\end{small}
\vspace{0ex}
\end{table}

%% file: figures/llm-utility.tex
\begin{table}[!t]
\centering
\small %
\renewcommand{\arraystretch}{0.9}
\setlength{\tabcolsep}{3pt} %
\caption{Impact of $\Delta^\text{\scriptsize{LLM}}_i$ on different datasets.}
\label{tab:xi-ablation-wide}
\vspace{-2ex}
\begin{tabular}{l cc cc cc cc}
\toprule
\multirow{2}{*}{\textbf{Setting}} & \multicolumn{2}{c}{\textit{Cora}} & \multicolumn{2}{c}{\textit{Alaska}} & \multicolumn{2}{c}{\textit{Music}} & \multicolumn{2}{c}{\textit{Movies}} \\
\cmidrule(lr){2-3} \cmidrule(lr){4-5} \cmidrule(lr){6-7} \cmidrule(lr){8-9}
 & \textbf{FP} & \textbf{NMI} & \textbf{FP} & \textbf{NMI} & \textbf{FP} & \textbf{NMI} & \textbf{FP} & \textbf{NMI} \\
\midrule
Computed $\Delta$ & 86.21 & 89.74 & 78.72 & \textbf{90.11} & 76.63 & \textbf{88.56} & 63.84 & 76.99 \\
Fixed ($0.95$)    & \textbf{86.34} & \textbf{90.02} & \textbf{79.58} & 89.61 & \textbf{77.26} & 88.05 & \textbf{64.68} & \textbf{77.11} \\
\bottomrule
\end{tabular}
\vspace{-1ex}
\end{table}

%% file: figures/supervised.tex
\begin{table}[!t]
\centering
\small
\caption{Comparisons with supervised PLM-based methods (1:9 data split).}
\label{tab:supervised_comparison_wide}
\vspace{-2ex}
\setlength{\tabcolsep}{3pt} %
\renewcommand{\arraystretch}{1.1}
\begin{tabular}{l cc cc cc cc}
\toprule
\multirow{2}{*}{\textbf{Method}} & \multicolumn{2}{c}{\textit{Cora}} & \multicolumn{2}{c}{\textit{Alaska}} & \multicolumn{2}{c}{\textit{Music}} & \multicolumn{2}{c}{\textit{Movies}} \\
\cmidrule(lr){2-3} \cmidrule(lr){4-5} \cmidrule(lr){6-7} \cmidrule(lr){8-9}
 & \textbf{FP} & \textbf{NMI} & \textbf{FP} & \textbf{NMI} & \textbf{FP} & \textbf{NMI} & \textbf{FP} & \textbf{NMI} \\
\midrule

\texttt{DeepMatcher} & 67.19 & 72.56 & 53.91 & 61.75 & 50.69 & 59.81 & 42.17 & 52.79 \\
\texttt{Ditto}       & \underline{70.66} & \underline{77.67} & \underline{62.39} & \underline{71.49} & \underline{52.84} & \underline{61.49} & \underline{44.21} & \underline{55.78} \\
\algo{}              & \textbf{86.34} & \textbf{90.02} & \textbf{79.58} & \textbf{89.61} & \textbf{77.26} & \textbf{88.05} & \textbf{64.68} & \textbf{77.11} \\
\bottomrule
\end{tabular}
\vspace{-1ex}
\end{table}

%% file: tex/prompt.tex
\begin{figure}
\centering
\small
\begin{tcolorbox}[title=\texttt{Entity Resolution}]
\textbf{Instruction}

\texttt{You are an entity resolution expert. Determine if the following entity matches any of the candidate entities.\\}

\texttt{Requirements:}

\texttt{1. Please respond with ONLY the number (1, 2, ...) of the matching candidate, or "NONE" if no match is found.}

\texttt{2. Your response should be a single number or "NONE", nothing else.}

\tikz \draw[dashed] (0,0) -- (\linewidth,0);

\textbf{Input Prompt}

\texttt{Query Entity:\\}
\texttt{Intermezzo in B minor, Op. 119 No. 1: Adago - Brahms: Complete Works\\}

\texttt{Candidate Entities:\\}
\texttt{1. Opus 6 No. 12 in B minor (HWV 330) - I. Largo - Concerti Grossi op. 6\\}
\texttt{2. 003-Symphony 1 in C minor, op. 11: III. Menuetto \& Trio, Allegro di Molto\\}
\texttt{3. Intermezzo in B minor, Op. 119 No. 1: Adagio\\}
\texttt{4. 017-Intermezzo in B minor, Op. 119 No. 1: Adagio\\}
\texttt{5. Johannes Brahms - Concerto for Violin and Orchestra in D major, Op. 77: II. Adagio}

\tikz \draw[dashed] (0,0) -- (\linewidth,0);

\textbf{Response}

\texttt{3}
\end{tcolorbox}
\caption{An example of the prompt used for LLM verification on \textit{Music}.}\label{fig:entity_resolution_prompt}
\end{figure}